\documentclass{article}

    \PassOptionsToPackage{numbers, compress}{natbib}

 \usepackage[main, final]{neurips_2026}

\usepackage[utf8]{inputenc} 
\usepackage[T1]{fontenc}    
\usepackage{hyperref}       
\usepackage{url}            
\usepackage{booktabs}       
\usepackage{amsfonts}       
\usepackage{nicefrac}       
\usepackage{microtype}      
\usepackage{xspace}
\usepackage{amsmath,amssymb,amsthm}
\usepackage{mathtools} 
\usepackage{adjustbox}
\usepackage{titletoc}
\usepackage[dvipsnames]{xcolor}
\usepackage{tcolorbox}                
\newtheorem{theorem}{Theorem}
\newtheorem{proposition}{Proposition}
\newtheorem{lemma}{Lemma}

\theoremstyle{definition}

\usepackage{multirow}

\makeatletter
\@ifpackageloaded{xspace}{
  \xspaceaddexceptions{"'}
}{}
\makeatother
\newcommand*{\method}{SymDrift\xspace}

\newcommand{\withgeom}[1]{}

\usepackage{algorithm}
\usepackage{algpseudocode}

\title{SymDrift: One-Shot Generative Modeling under Symmetries}

\author{%
  Samir Darouich\thanks{Equal contribution.}\\
  Institute of Artificial Intelligence\\
  Institute of Theoretical Chemistry\\
  University of Stuttgart\\
  \texttt{samir.darouich@ki.uni-stuttgart.de} \\
  \And
  Vinh Tong\footnotemark[1]\\
  Institute of Artificial Intelligence\\
  University of Stuttgart\\
  \texttt{vinh.tong@ki.uni-stuttgart.de} \\
  \And
  Lluís Pastor-Pérez \\
  Institute of Artificial Intelligence\\
  University of Stuttgart\\
  \texttt{lluis.pastor-perez@ki.uni-stuttgart.de} \\
  \And
  Tanja Bien \\
  Institute of Artificial Intelligence\\
  University of Stuttgart\\
  \texttt{tanja.bien@ki.uni-stuttgart.de} \\
  \And
  Loay Mualem\\
  Institute of Artificial Intelligence\\
  University of Stuttgart\\
  \texttt{loay.mualem@ki.uni-stuttgart.de} \\
  \And
  Mathias Niepert\\
  Institute of Artificial Intelligence\\
  University of Stuttgart\\
  \texttt{mathias.niepert@ki.uni-stuttgart.de} \\
}

\begin{document}

\maketitle

\begin{abstract}

Generative modeling of physical systems, such as molecules, requires learning distributions that are invariant under global symmetries, such as rotations in three-dimensional space. Equivariant diffusion and flow matching models can incorporate such invariances effectively, even when trained on a non-invariant empirical distribution, but they typically rely on costly multi-step sampling. Recently, drifting models have emerged as an efficient alternative, enabling single-step generation and achieving state-of-the-art performance in generative modeling tasks. However, we show that drifting models face a symmetry-specific challenge, since an equivariant generator does not generally produce the same drifting field as the one obtained from the symmetrized target distribution. Addressing this issue would require expensive symmetrization of the empirical distribution. To avoid this cost, we propose \method, a framework that makes the drifting field itself symmetry-aware. We introduce two complementary strategies: (i) a symmetrized drift in coordinate space based on optimal alignment, and (ii) a $G$-invariant embedding that removes symmetry ambiguity by construction. Empirically, \method outperforms existing one-shot methods on standard benchmarks for conformer and transition state generation, while remaining competitive with significantly more expensive multi-step approaches. By enabling one-shot inference, \method reduces computational overhead by up to 40$\times$ compared to existing baselines, making it promising for high-throughput applications such as virtual drug screening and large-scale reaction network exploration. The code is available at \href{https://github.com/samirdarouich/SymDrift}{github.com/samirdarouich/SymDrift}.

\end{abstract}
\section{Introduction}

Generative models~\cite{ho2020denoising,song2021scorebased, deng2026generative,lipman2023flow,liu2023flow,albergo2025stochastic} have emerged as powerful tools for learning complex distributions, with applications spanning images~\cite{karras2022elucidating,geng2025mean,zheng2026diffusion}, molecules~\cite{zeng2026propmolflow,song2023equivariant,klein2023equivariant,cao2025efficient, tong2024improving}, and proteins~\cite{abramson2024accurate,geffner2025proteina}. In many physical domains, the underlying data distribution exhibits inherent symmetries, for instance, molecular properties are invariant under rotations in $\mathbb{R}^3$ and permutations of identical atoms. Formally, this setting is captured by distributions $p$ that are invariant under the action of a symmetry group $G$, i.e., $p(\mathbf{x}) = p(g\mathbf{x})$ for all $g \in G$~\cite{kohler2020equivariant}. Learning such $G$-invariant distributions poses unique challenges, as multiple configurations correspond to the same physical state, leading to ambiguity in representation~\cite{hoogeboom2022equidiff}. 

Recent work addresses this challenge by incorporating symmetry into generative models through equivariant architectures and symmetry-aware transport~\cite{song2023equivariant,klein2023equivariant,cao2025efficient,hoogeboom2022equidiff,tong2025raoblackwell}. In flow-based models, the choice of transport path plays a central role. In particular, straight paths induced by optimal transport couplings improve efficiency by reducing trajectory curvature and the number of function evaluations~\cite{tong2024improving}. In symmetric domains, however, defining such paths requires accounting for equivalence under $G$. Distances between configurations can be misleading if symmetry is ignored, leading to suboptimal transport. Prior work addresses this by aligning samples before computing pairwise distances, enabling equivariant transport maps that are shorter and more stable~\cite{song2023equivariant,klein2023equivariant,hassan2024flow}.

Despite these advances, diffusion and flow-based models rely on multi-step sampling procedures, making them computationally expensive in high-throughput settings such as virtual screening. This has motivated the development of more efficient generative frameworks~\cite{liu2023flow,song2023consistency,tong2025learning}, ultimately leading to one-step models~\cite{geng2025mean, liu2023instaflow}. A common strategy is to distill multi-step dynamics into a single forward pass, for example via ReFlow~\cite{liu2023flow}. While effective, these approaches remain tied to a pre-defined generative trajectory and typically require multi-stage training, additional losses, and careful tuning~\cite{kim2025simple,zhu2024slimflow,lee2024improving}.

Recently, \citet{deng2026generative} introduced drifting models, a direct alternative to iterative and distillation-based approaches. Instead of learning and compressing a multi-step generative trajectory, drifting is, by design, a one-step model that formulates generation as a distribution-matching problem by learning a vector field whose stationary point matches the target distribution. However, drifting models implicitly assume access to a well-defined target distribution. In symmetric domains, this assumption breaks, as training is performed on a finite empirical distribution that is not invariant under $G$. 

\begin{figure}[t]
\centering
\includegraphics[width=\textwidth]{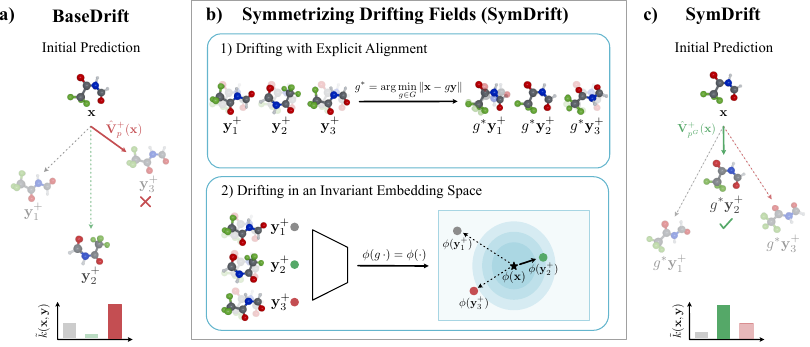}
\label{fig:method_overview}
\caption{
\textbf{Conceptual illustration of \method}.
\textbf{a} The base drift assigns highest kernel weights $\tilde{k}(\mathbf{x},\mathbf{y})$ to $\mathbf{y}^+_3$, despite $\mathbf{y}^+_2$ being the closest conformer under symmetry, leading to a biased drifting field $\hat{\mathbf{V}}_{p}^{+}(\mathbf{x}) \neq \mathbf{V}_{p^{G}}^{+}(\mathbf{x})$. 
\textbf{b} \method incorporates symmetries via 1) explicit alignment or 2) $G$-invariant embedding space.
\textbf{c} This leads to \method correctly identifying $\mathbf{y}^+_2$ in the symmetrized space, producing the drifting field $\hat{\mathbf{V}}_{p^{G}}^{+}(\mathbf{x}) \approx \mathbf{V}_{p^{G}}^{+}(\mathbf{x})$.
}
\end{figure}

To address this limitation, we propose \method, a symmetry-aware extension of drifting for generative modeling on highly symmetric distributions. Our key observation is that the learned drifting field is biased toward arbitrary representatives of each equivalence class, preventing recovery of the drift associated with the symmetrized target distribution. To resolve this, we explicitly account for symmetry during training and propose two complementary strategies: (i) aligning configurations in coordinate space to construct a symmetrized drift, and (ii) learning in an invariant representation that removes symmetry by construction. Our main contributions are summarized as follows:

\begin{itemize}
    \item We introduce \method, a framework for drifting-based one-shot generative modeling of highly symmetric distributions.
    \item We develop two complementary strategies to account for symmetry: (i) a symmetrized drift in coordinate space via alignment, and (ii) an invariant embedding that removes symmetry by construction.
    \item We achieve state-of-the-art one-shot performance on conformer and transition state generation benchmarks, while remaining competitive with iterative baselines and reducing computational cost by up to $40\times$.
\end{itemize}
\section{Related work}

\paragraph{Equivariance in generative modeling.}

In scientific domains, incorporating symmetry-aware inductive biases improves data efficiency and generalization. In particular, equivariant architectures ensure that generative models respect physical symmetries in 3D space. Equivariant diffusion models~\cite{kohler2020equivariant,hoogeboom2022equidiff} and subsequent flow matching approaches~\cite{song2023equivariant,klein2023equivariant,cao2025efficient} leverage this property to recover symmetry-consistent distributions, even when trained on finite, non-invariant data.

\paragraph{One-step generation.}

While equivariant diffusion and flow matching models provide strong symmetry guarantees, they rely on costly iterative sampling. This has motivated the development of one-step generative frameworks that eliminate inference-time iteration. Approaches such as Rectified Flow (ReFlow)~\cite{liu2023flow}, consistency models~\cite{song2023consistency}, and mean flow models~\cite{geng2025mean} approximate the probability flow with a direct mapping, often via distillation or trajectory consistency. Recently, \citet{deng2026generative} proposed drifting models, a single-step generative framework grounded in direct distribution matching. At its core, the approach minimizes the drifting field governed by dynamics that naturally align with the target data distribution at their equilibrium state. By ensuring the target distribution serves as the stationary point of this field, the model can map inputs directly to the desired data manifold in a single forward pass.

\paragraph{Molecular structure generation.}

We consider two central tasks in molecular modeling: conformer and transition-state (TS) generation. In conformer generation, the goal is to produce 3D molecular structures consistent with a given molecular graph. While traditionally addressed using rule-based methods such as RDKit~\cite{riniker2015better}, recent work has focused on data-driven approaches, including diffusion and flow-based models~\cite{luo2021predicting,shi2021learning,xu2022geodiff,jing2022torsional,liu2025nextmol}. More recently, one-shot methods such as AvgFlow~\cite{cao2025efficient} and EnFlow~\cite{xu2025energy} have been proposed to reduce inference cost. In TS generation, the goal is to predict the saddle-point geometry of a reaction from the molecular graphs of the reactants and products. Classical methods rely on iterative optimization or path-based techniques~\cite{jonsson1998nudged,henkelman_climbing_2000,banerjee_search_1985,henkelman_dimer_1999}, which become prohibitively expensive when combined with high-level electronic structure calculations. Recent machine learning approaches, including equivariant diffusion and flow-based models~\cite{pattanaik2020generating,duan2023accurate,yuan_analytical_2024,wander_cattsunami_2024,kim2024diffusion,galustian2025goflow,darouich2026adaptive,galustian2026motsart}, have improved efficiency, but still require iterative inference.
\section{Background}

Drifting~\cite{deng2026generative} is a one-shot generative model that learns a mapping from a prior distribution $p_\epsilon$ to a target distribution $p$. Specifically, a generator $f_\theta$ transforms samples $\boldsymbol{\epsilon} \sim p_\epsilon$ from the prior distribution into samples $\mathbf{x} = f_\theta(\boldsymbol{\epsilon})$, inducing an output distribution \(q_\theta = (f_\theta)_\# p_\epsilon\)
with $\#$ as the pushforward operator and $q_\theta \approx p$ as the objective. Training is formulated as a fixed-point iteration, where the generator in each training step $i$ is updated according to
\begin{equation}
    f_{\theta_{i+1}}(\boldsymbol\epsilon) \leftarrow f_{\theta_{i}}(\boldsymbol\epsilon) + \mathbf{V}_{p,q_{\theta_{i}}}(f_{\theta_{i}}(\boldsymbol\epsilon))
\end{equation}
and $\mathbf{V}_{p,q_{\theta_i}}$ denotes the drifting field governing the evolution of samples $\mathbf{x}$ during training. Convergence is achieved when the generated distribution matches the target distribution, i.e., $q_\theta = p$. At this point, the drifting field vanishes, $\mathbf{V}_{p,q_\theta}(\mathbf{x}) = \mathbf{0}$ for $\mathbf{x} \sim q_\theta$, implying that samples are stationary under the update and the fixed-point condition is satisfied. To convert this fixed-point update into a trainable loss, the stop-gradient (sg) operation is applied, which treats the target as a frozen reference:
\begin{equation}
    \label{eq:drifting_loss}
    \mathcal{L} 
    = \mathbb{E}_{\boldsymbol{\epsilon} \sim \mathcal{N}(0,I)} 
      \left[ \left\| f_\theta(\boldsymbol{\epsilon}) 
      - \operatorname{sg}\!\bigl(f_\theta(\boldsymbol{\epsilon}) 
      + \mathbf{V}_{p,q_{\theta}}(f_\theta(\boldsymbol{\epsilon}))\bigr) \right\|^{2} \right].
\end{equation}
This loss encourages the network to move its predictions toward the fixed-point defined by the drifting field, which is defined as a combination of attractive and repulsive components:
\begin{equation}
    \mathbf{V}_{p,q}(\mathbf{x}) = \mathbf{V}^{+}_{p}(\mathbf{x}) - \mathbf{V}^{-}_{q}(\mathbf{x}).
\end{equation}
Here, the positive term $\mathbf{V}^{+}_p$ draws samples toward the target distribution $p$, while the negative term $\mathbf{V}^{-}_q$ repels samples according to the current output distribution $q$, promoting diversity and mitigating mode collapse. The drift for a single sample $\mathbf{x}$ is computed as a kernel-weighted average over all points in the target and current distributions:
\begin{align}
    \mathbf{V}^{+}_{p}(\mathbf{x}) &= \frac{1}{Z_p} \mathbb{E}_{\mathbf{y^{+}} \sim p} \left[ k(\mathbf{x},\mathbf{y^{+}}) (\mathbf{y^{+}}-\mathbf{x}) \right] & \qquad
    \mathbf{V}^{-}_{q}(\mathbf{x}) &= \frac{1}{Z_q} \mathbb{E}_{\mathbf{y^{-}} \sim q
} \left[ k(\mathbf{x},\mathbf{y^{-}}) (\mathbf{y^{-}}-\mathbf{x}) \right]
\end{align}
where the normalization constants are $Z_p = \mathbb{E}_{\mathbf{y^{+}} \sim p}\left[k(\mathbf{x},\mathbf{y^{+}})\right]$ and $Z_q=\mathbb{E}_{\mathbf{y^{-}} \sim q} \left[k(\mathbf{x},\mathbf{y^{-}})\right]$. Following \citet{deng2026generative}, the kernel $k$ is chosen as the Laplacian kernel:
\begin{equation}
    k(\mathbf{x},\mathbf{y}) = \text{exp}\left( \frac{-\|\mathbf{x}-\mathbf{y}\|}{\tau} \right)
\end{equation}
where $|| \cdot ||$ is the $\ell_2$ distance and $\tau$ the temperature. The normalized kernel $\tilde{k}(\mathbf{x},\mathbf{y}) \coloneq \frac{1}{Z} k(\mathbf{x},\mathbf{y})$ can be efficiently obtained by applying a softmax over the logits $-\frac{1}{\tau}\|\mathbf{x}-\mathbf{y}\|$ across all samples $\mathbf{y}$. Alternatively, a two-sided normalization can be obtained by additionally applying the softmax over $\mathbf{x}$ and combining both normalizations via their geometric mean.

During mini-batch training, $N_\text{c}$ classes are sampled, and for each class, $N^{+}$ samples are drawn from the dataset and $N^{-}$ samples from the prior distribution. The prior samples are passed through the generator and serve both as the current outputs $\mathbf{x}$ and as negative samples $\mathbf{y}^{-}$. The drift for a single sample $\mathbf{x}_i$ is then computed explicitly as
\begin{equation}
    \label{eq:drift_explicit}
    \mathbf{V}_{p,q}(\mathbf{x}_i) = \sum_j^{N^{+}} \tilde{k}(\mathbf{x}_i,\mathbf{y}^{+}_j) (\mathbf{y}^{+}_j-\mathbf{x}_i) - \sum_k^{N^{-}}  \tilde{k}(\mathbf{x}_i,\mathbf{y}^{-}_k) (\mathbf{y}^{-}_k-\mathbf{x}_i).
\end{equation}
Unlike diffusion or flow matching approaches, this formulation explicitly considers all data and generated points for each sample, providing a direct mechanism for matching the distributions. The final training loss is then defined as the expectation of the drifting loss from Equation~\ref{eq:drifting_loss} over all generated samples. Inference is performed in a single forward pass by sampling from the prior and evaluating the generator, in contrast to iterative procedures required by diffusion or flow-based models.
In addition, \citet{deng2026generative} propose to compute the drift in an embedding space using a feature extractor $\phi$:
\begin{equation}
    \label{eq:drift_loss_embedded}
    \mathcal{L}_{\phi} = \mathbb{E}_{\boldsymbol{\epsilon}} \left[ \left\| \phi(f_{\theta}(\boldsymbol{\epsilon})) - \text{sg}\left(\phi(f_{\theta}(\boldsymbol{\epsilon})) + \mathbf{V}_{p,q_{\theta}}(\phi(f_{\theta}(\boldsymbol{\epsilon}) ))\right) \right\|^2 \right].
\end{equation}
This embedding plays a crucial role, as the drifting field relies on a kernel to measure similarities between samples. By mapping inputs to a feature space where semantically similar configurations are close, the embedding ensures that the resulting drift reflects meaningful structure in the data. Section~\ref{sec:extended_background} provides additional background, including a discussion of equivariance and invariance.
\section{Drifting with Equivariant Generators}

In this section, we introduce \method, a framework for training drift-based one-shot generative models in a symmetry-aware manner. We first discuss the challenges that arise when learning symmetric distributions with drift-based models, and then present our approach to address them.

\subsection{The Challenge of Symmetrized Drifting}\label{sec:equi_generative}

Target distributions arising from physical systems are often $G$-invariant. In molecular systems, these symmetries are described by a group $G$ that models rotations, translations, and permutations. In practice, however, we only observe a finite dataset $\mathcal{D}$, which induces an empirical distribution $p$. Since $\mathcal{D}$ is typically not closed under the group action, $p$ is generally not $G$-invariant. A standard way to encourage symmetry in the learned distribution is to augment the data with transformed samples. Ideally, this corresponds to training against the fully symmetrized empirical distribution $p^G$, while in practice, finite augmentation only approximates this group average. 

For diffusion and flow matching models, $G$-equivariant neural networks provide an alternative mechanism for incorporating symmetry, eliminating the need for explicit data augmentation ~\cite{song2023equivariant,klein2023equivariant,kohler2020equivariant,hoogeboom2022equidiff}. In this setting, optimizing the training objective under $p$ is mathematically equivalent to optimizing it under $p^G$~\cite{tong2025raoblackwell,chen2024equivariant}. For drifting models, however, we show that the resulting drifting fields obtained from $p$ and $p^G$ differ, even when using an equivariant architecture with the unsymmetrized empirical distribution.

To this end, we define the drifting objective and its gradient with respect to $\theta$:
\begin{align}
\mathcal{L}(\theta) &= \mathbb{E}_{\boldsymbol{\epsilon} \sim p_{\epsilon}} \left\| \mathbf{x} - \operatorname{sg}\!\bigl(\mathbf{x} + \mathbf{V}^+_p(\mathbf{x}) - \mathbf{V}^-_{q_\theta}(\mathbf{x})\bigr) \right\|^2, \\
\nabla_\theta \mathcal{L} &= \mathbb{E}_{\boldsymbol{\epsilon} \sim p_{\epsilon}} \left[ -2 \bigl(\nabla_\theta f_\theta(\boldsymbol{\epsilon})\bigr)^\top \bigl(\mathbf{V}^+_p(\mathbf{x}) - \mathbf{V}^-_{q_\theta}(\mathbf{x})\bigr) \right]
\end{align}
with $\mathbf{x} = f_\theta(\boldsymbol{\epsilon})$ as the model prediction. Because the prior distribution $p_\epsilon$ is $G$-invariant and $f_\theta$ is equivariant, such that $f_\theta(g\boldsymbol{\epsilon}) = g f_\theta(\boldsymbol{\epsilon})$, the induced model distribution $q_\theta = (f_\theta)_\# p_\epsilon$ is also $G$-invariant~\cite{kohler2020equivariant,hoogeboom2022equidiff}. This invariance naturally allows expectations with respect to $q_\theta$ to be equivalently expressed as averages over the group, which yields:
\begin{equation}
\label{eq:marginalized_grad}
    \nabla_\theta \mathcal{L} = \mathbb{E}_{\boldsymbol{\epsilon} \sim p_{\epsilon}} \mathbb{E}_{g \sim G} \left[ -2 \bigl(\nabla_\theta f_\theta(g\boldsymbol{\epsilon})\bigr)^\top \bigl(\mathbf{V}^+_p(g\mathbf{x}) - \mathbf{V}^-_{q_\theta}(g\mathbf{x})\bigr) \right].
\end{equation}
A direct consequence of the $G$-invariance of $q_\theta$ is that the resulting negative drifting vector field $\mathbf{V}^-{q_\theta}$ is $G$-equivariant, i.e., $\mathbf{V}^-{q_\theta}(g\mathbf{x}) = g\,\mathbf{V}^-{q_\theta}(\mathbf{x})$. We refer to Section~\ref{proof_property1} for a formal proof. Moreover, using the orthogonality of $G$, it follows that
$(\nabla_\theta f_\theta(g\boldsymbol{\epsilon}))^\top = (\nabla_\theta f_\theta(\boldsymbol{\epsilon}))^\top g^{-1}$.
Substituting these relations into the gradient yields:
\begin{equation}
\label{eq:conflict_grad}
    \nabla_\theta \mathcal{L}
    = \mathbb{E}_{\boldsymbol{\epsilon} \sim p_{\epsilon}} \Big[
    -2\, \bigl(\nabla_\theta f_\theta(\boldsymbol{\epsilon})\bigr)^\top
    \big(\mathbb{E}_{g \sim G} \bigl[g^{-1} \mathbf{V}^+_p(g\mathbf{x})\bigr]
    - \mathbf{V}^-_{q_\theta}(\mathbf{x})
    \big)
    \Big].
\end{equation}
We refer the reader to Section~\ref{sec:gradient_def} for a detailed derivation.
Crucially, the aggregated vector field $\hat{\mathbf{V}}_{p}^+(\mathbf{x}) = \mathbb{E}_g \bigl[g^{-1} \mathbf{V}^+_p(g\mathbf{x})\bigr]$ is generally not equal to $\mathbf{V}^+_{p^G}(\mathbf{x})$, the vector field induced by the symmetrized distribution:
\begin{equation}
    \label{eq:sym_drift_field}
    \mathbf{V}^+_{p^G}(\mathbf{x}) = \frac{\mathbb{E}_{\mathbf{y}^+ \sim p^G}[k(\mathbf{x}, \mathbf{y}^+)(\mathbf{y}^+ - \mathbf{x})]}{\mathbb{E}_{\mathbf{y}^+ \sim p^G} [k(\mathbf{x},\mathbf{y}^+)]} = \frac{\mathbb{E}_{\mathbf{y}^+ \sim p} \mathbb{E}_{g \sim G} [k(\mathbf{x}, g\mathbf{y}^+)(g\mathbf{y}^+ - \mathbf{x})]}{\mathbb{E}_{\mathbf{y}^+ \sim p}\mathbb{E}_{g \sim G} [k(\mathbf{x}, g\mathbf{y}^+)]}.
\end{equation}

This mismatch shows that, for drifting models, equivariance of the generator alone does not generally reproduce the drift induced by the symmetrized empirical distribution. We formalize this finding in Theorem~\ref{thm:sym_problem}. 

\begin{tcolorbox}[title=Drift field mismatch under symmetrization]
\begin{theorem}
\label{thm:sym_problem}
Let \(G\) be a compact orthogonal group on \(\mathbb{R}^d\) with Haar measure, and define
\[
\mathbf{V}_p^+(\mathbf{x})
=
\frac{\mathbb{E}_{\mathbf{y}^+ \sim p}\!\left[k(\mathbf{x},\mathbf{y}^+)(\mathbf{y}^+-\mathbf{x})\right]}
     {\mathbb{E}_{\mathbf{y}^+ \sim p}\!\left[k(\mathbf{x},\mathbf{y}^+)\right]},
\]
where 
\(
k(\mathbf{x},\mathbf{y}^+)=\exp\!\left(-\frac{\|\mathbf{x}-\mathbf{y}^+\|}{\tau}\right)
\)
is the Laplacian kernel.

Let \(p^G(\mathbf{x})=\mathbb{E}_{g\sim G}[p(g^{-1}\mathbf{x})]\) and
\(
\hat{\mathbf{V}}_{p}^+(\mathbf{x})=\mathbb{E}_{g\sim G}\!\left[g^{-1}\mathbf{V}_p^+(g\mathbf{x})\right].
\)
Then there exist a non-\(G\)-invariant \(p\) and \(\mathbf{x} \in \mathbb{R}^d\) such that
\(
\hat{\mathbf{V}}_{p}^+(\mathbf{x}) \neq \mathbf{V}_{p^G}^+(\mathbf{x}).
\)
\end{theorem}
\end{tcolorbox}

We provide a full proof in Section \ref{appendix: main_thm}. Consequently, when using equivariant generators in drift-based models, both the drifting field and the loss must be defined with respect to the symmetrized target distribution $p^G$. This, in turn, makes the estimation of expectations in Equation~\ref{eq:sym_drift_field} computationally demanding, as it requires averaging over a large number of group-augmented samples, even when the architecture itself is equivariant. As a result, the sample-efficiency gains typically associated with equivariant networks are largely diminished, weakening the original motivation for their use in this setting and motivating the consideration of more flexible, non-equivariant architectures. Crucially, this issue is drastically increased in the low-temperature regime, which is required in drift-based models in order to generate sharp distributions~\cite{deng2026generative}. As the kernel becomes increasingly peaked, however, most transformed samples $g\mathbf{y}$ fall outside its effective support, contributing negligibly to the expectation due to the exponential decay of the Laplacian kernel. In this regime, the group integral effectively behaves as a soft minimum over transformations. As $\tau \to 0$, the probability mass concentrates on those $g \in G$ that map $g\mathbf{y}$ closest to $\mathbf{x}$, causing the smooth expectation over $G$ to collapse into a hard alignment. In the limit, only the transformation that best aligns $\mathbf{x}$ and $\mathbf{y}$ contributes (see Section~\ref{hard_alignment_lim} for a formal proof).

\subsection{Method 1: Drifting with Explicit Alignment}
\label{sec:method_alignment}

Motivated by this limiting behavior, we can sidestep the prohibitive computational cost of explicit data augmentation by designing a symmetry-aware drifting field. Concretely, for each generated sample $\mathbf{x}$ and target sample $\mathbf{y}$, we first remove the center of mass and then find the optimal group transformation $g^*(\mathbf{x}, \mathbf{y})$ that minimizes their distance:
\begin{equation}
    \label{eq:sym_aware_diff}
    g^*(\mathbf{x}, \mathbf{y}) = \arg\min_{g \in G} \left\| \mathbf{x} - g\mathbf{y} \right\|.
\end{equation}
We then construct our revised drifting field by replacing the standard pairs with their optimally aligned counterparts, substituting $(\mathbf{x}_i, g^*(\mathbf{x}_i, \mathbf{y}_j)\mathbf{y}_j)$ into Equation~\ref{eq:drift_explicit}. This explicit alignment effectively circumvents the need to average over uniformly sampled group elements. By forcing the target to adopt the optimal orientation relative to the generated sample, it provides a practical approximation of the symmetrized drifting field $\hat{\mathbf{V}}_{p^{G}}^{+}(\mathbf{x}) \approx \mathbf{V}^+_{p^G}(\mathbf{x})$ in the low-temperature regime. Importantly, this construction also yields a $G$-equivariant drifting field, as shown in Section~\ref{sec:equi_drift_alignment}.

While the optimal rotation $R^*$ can be computed efficiently via the Kabsch algorithm~\cite{kabsch1976solution}, determining the optimal permutation $\Pi^*$ remains challenging. A brute-force search over permutations scales factorially in the number of identical atoms and is therefore intractable. However, even approximate strategies introduce significant overhead. A common approach alternates between estimating $\Pi^*$ using the Hungarian~\cite{kuhn1955hungarian} algorithm and updating $R^*$ via Kabsch~\cite{klein2023equivariant}. Although the Hungarian algorithm reduces the combinatorial complexity, it still scales cubically in the number of atoms. This makes the iterative alignment procedure computationally expensive, while also being prone to local minima and lacking guarantees of global consistency.

\subsection{Method 2: Drifting in an Invariant Embedding Space}
\label{sec:route_invariant}

An alternative approach is to compute the drift in a fixed embedding space instead of the coordinate space, as shown in Equation~\ref{eq:drift_loss_embedded}. For this construction to provide a meaningful learning signal, the feature extractor must be sufficiently expressive~\cite{deng2026generative}. Otherwise, semantically distinct inputs may collapse in the embedding space, causing the kernel $k(\cdot,\cdot)$ to assign high similarity to configurations that differ in structurally relevant ways. However, avoiding collapse is only part of the requirement, since the embedding must also faithfully reflect the geometry of the domain. Specifically, if the feature representation does not explicitly account for the symmetries of the underlying physical system, the resulting distance metric in this space may misrepresent the true physical differences between states and thereby render the learning signal unreliable.

\paragraph{Invariant embedding construction.}

To incorporate symmetry into the representation, we construct a $G$-invariant embedding based on type-conditioned pairwise distances. Let $\mathbf{x} = (x_1,\dots,x_N) \in \mathbb{R}^{N\times 3}$ denote a molecular configuration, where each atom $x_i \in \mathbb{R}^3$ has a type $t_i \in \mathcal{T}$. For a pair $(i,j)$, define its interaction type as $\tau(i,j) := \{t_i,t_j\}$, and let $\mathcal{T}_{\mathrm{pairs}} = \{\{a,b\} \mid a,b \in \mathcal{T}\}$. For each interaction type \(\alpha \in \mathcal{T}_{\mathrm{pairs}}\), we collect the corresponding pairwise distances and concatenate the sorted distance sets:
\begin{equation}
    \label{eq:sorting}
    \mathcal{D}_\alpha(\mathbf{x})
    :=
    \{ \|x_i - x_j\| \mid 1 \le i < j \le N,\; \tau(i,j)=\alpha \},
    \qquad 
    \phi(\mathbf{x})
    =
    \bigoplus_{\alpha \in \mathcal{T}_{\mathrm{pairs}}}
    \operatorname{sort}\big(\mathcal{D}_\alpha(\mathbf{x})\big).
\end{equation}
By construction, \(\phi\) depends only on pairwise distances and is therefore invariant to global rotations ($SO(3)$). Grouping distances by interaction type and sorting within each group further removes dependence on permutations of atoms within the same species ($\hat S_N$). As a result, \(\phi\) is invariant under the action of \(G = SO(3)\times \hat S_N\).

\paragraph{Structural properties.}

The embedding $\phi$ is constructed as a symmetry-invariant representation over typed point configurations with distance structure. It is therefore invariant by construction and efficient to compute. We summarize the properties of the embedding in Proposition~\ref{prop:invariance}.

\begin{tcolorbox}[title=Invariance and Regularity]
\label{prop:invariance}
\begin{proposition}
Let \(\mathcal{X}_N = \{\mathbf{x} \in \mathbb{R}^{N\times 3} : \sum_i x_i = 0\}\), where each \(x_i\) has type \(t_i \in \mathcal{T}\). For \(\alpha \in \mathcal{T}_{\mathrm{pairs}}\), define
\[
\mathcal{D}_\alpha(\mathbf{x}) = \{ \|x_i - x_j\| : \{t_i,t_j\}=\alpha \},
\quad
\phi(\mathbf{x}) = \bigoplus_\alpha \operatorname{sort}(\mathcal{D}_\alpha(\mathbf{x})).
\]
Let \(G = SO(3)\times \hat S_N\). Then:

\textbf{(i) Invariance.} \(\phi\) is \(G\)-invariant:
\(
\phi(\mathbf{x}) = \phi(g\mathbf{x}), \quad \forall g \in G.
\)

\textbf{(ii) Regularity.} \(\phi\) is piecewise smooth and differentiable almost everywhere, with non-differentiability only when pairwise distances within a type class coincide.
\end{proposition}
\end{tcolorbox}
A desirable property of an invariant embedding is injectivity, which ensures that two point clouds not related by a symmetry do not map to the same representation. Our embedding $\phi$ does not guarantee injectivity, since it retains only type-resolved distance multisets and discards incidence information, i.e., which distances are associated with which atom pairs. This choice provides a practical trade-off between expressivity and computational efficiency. Importantly, however, our framework is not tied to a particular choice of $\phi$. Any differentiable invariant embedding can be used to define the drifting field, which includes embeddings based on $\phi$ with more fine-grained node types derived from a graph coloring algorithm or on more expressive geometric message-passing neural networks.

\paragraph{Computational complexity.}
The embedding is non-parametric and requires no training. Its computational cost is dominated by pairwise distance computations and sorting. For a point cloud with $N$ points, the worst-case complexity is $\mathcal{O}(N^2 \log N)$ and, in typical finite-precision implementations, effectively $\mathcal{O}(N^2)$. In contrast, the alignment-based method introduced in Section~\ref{sec:method_alignment} requires solving a permutation matching problem for each pair of configurations. In practice, this is typically approximated by alternating rotational alignment with type-wise assignment, whose Hungarian step has worst-case complexity $\mathcal{O}(N^3)$.

\section{Experiments}
\label{sec:experiments}

We empirically evaluate \method by comparing generated and ground-truth conformers and TSs using distance-based metrics. Implementation and metrics details are provided in Section~\ref{sec:implementation}.

\subsection{Molecular Conformer Generation}

\paragraph{Datasets and Metrics.}  


For conformer generation, we evaluate our models on the GEOM dataset~\cite{axelrod2022geom}, including GEOM-QM9 (133,258 small molecules) and GEOM-Drugs (304,466 drug-like molecules) which generated the reference conformers using CREST~\cite{pracht2024crest}. We adopt the standard data splits from~\citet{ganea2021geomol}, allocating 80\% of the data for training, 10\% for validation, and using a subset of 1,000 molecular graphs for testing. Model performance is evaluated using coverage (COV) and average minimum root-mean-square deviation (AMR). A lower AMR indicates higher accuracy, while a higher COV reflects greater diversity. For both metrics, we report recall, measuring the extent to which the generated conformers capture the ground-truth conformers, and precision, which quantifies the proportion of generated conformers that are accurate. For the coverage metric, a generated conformer is considered correct if it lies within a threshold of $\delta = 0.5$~\AA\ for GEOM-QM9 and $\delta = 0.75$~\AA\ for GEOM-Drugs. Following prior work, we generate $2K$ conformers per molecular graph, where $K$ is the number of reference conformers, and apply chirality correction as introduced in~\citet{ganea2021geomol}.

\paragraph{Results.}

We consider two architectures for evaluating \method, namely the equivariant model introduced in ET-Flow~\cite{hassan2024flow} with 8.3M parameters, and the non-equivariant DiTMC model~\cite{frank2025sampling} with 8.9M (B) and 22.2M (L) parameters. Notably, by embedding symmetry awareness directly into the learned drift field, we eliminate the data augmentation that flow matching approaches typically require to approximate equivariance under non-equivariant parameterizations.

In Table~\ref{tab:performance_qm9} we report the results on the GEOM-QM9 dataset for both model architectures trained in the $G$-invariant embedding space using two-sided kernel normalization. The resulting models achieve state-of-the-art performance across all accuracy metrics among one-step inference baselines. To characterize the generated ensembles beyond the aggregate metrics, we analyze the distribution of pairwise heavy-atom RMSDs between conformers of the same molecule (Appendix~\ref{sec:add_experiments}). \method{} stays close to the reference ensemble and matches the multi-step ET-Flow baseline, showing that one-shot generation retains the structural diversity of iterative sampling. Two components of the drifting formulation govern how accurately the individual structures are recovered. The first is the number of negative samples $N^{-}$ drawn per molecular graph during training, which we find to consistently improve generation quality as it increases (Section~\ref{sec:add_experiments}). The second is the normalization of the drift kernel, whose weights are either normalized over the target samples alone or over both the generated and the target samples, where the two-sided variant gives a consistent although modest improvement across both architectures. Beyond the drifting formulation, the embedding itself introduces an approximation, since sorting the distances within an orbit discards their assignment to individual atoms and can map two structurally distinct conformers onto the same representation. How often this occurs depends on how the orbits are defined. Grouping the distances by atomic type treats all atoms of the same element as interchangeable, whereas deriving the orbits from the automorphisms of the molecular graph separates atoms of the same element that occupy chemically inequivalent positions. On GEOM-QM9, the first definition leaves $11.3\%$ of molecules with at least one aliased pair, while the second reduces this to $6.3\%$, without translating into a clear empirical difference between models trained with either (Section~\ref{sec:add_experiments}). A detailed description of both strategies is given in Appendix~\ref{sec:embedding_strategies}. In terms of efficiency, our method is up to 3.5$\times$ faster than EnFlow, which relies on an additional energy model for guidance. Furthermore, drifting is inherently a one-shot model and therefore does not require the additional reflow training stage needed by EnFlow to achieve competitive one-step accuracy. Even when compared to more computationally expensive settings such as ET-Flow, we achieve performance comparable to the 50 NFE configuration, while our one-shot generation reduces inference cost by a factor of 40. This highlights a clear trade-off between generation accuracy and computational efficiency, which we examine in more detail in Section~\ref{sec:train_details}.

\begin{table}[ht]
\centering
\caption{Molecule conformer generation results on GEOM-QM9 ($\delta$= 0.5\AA). “-R” denotes Recall and “-P” denotes Precision. Best in bold, second-best underlined; results separated into one-step and multi-step.}
\label{tab:performance_qm9}
\begin{adjustbox}{max width=\linewidth}
\begin{tabular}{lcccccccccc}
\toprule
& & Inference & \multicolumn{2}{c}{Coverage-R (\%) $\uparrow$} & \multicolumn{2}{c}{AMR-R (\AA) $\downarrow$} 
& \multicolumn{2}{c}{Coverage-P (\%) $\uparrow$} & \multicolumn{2}{c}{AMR-P  (\AA) $\downarrow$} \\
\cmidrule(lr){4-5} \cmidrule(lr){6-7} \cmidrule(lr){8-9} \cmidrule(lr){10-11}
Method & NFE & time (ms) & Mean & Median & Mean & Median & Mean & Median & Mean & Median \\
\midrule
GeoDiff~\cite{xu2022geodiff} & 5000 & $4.5 \times 10^{3}$ & 76.5 & 100.0 & 0.297 & 0.229 & 50.00 & 33.5 & 1.524 & 0.510 \\
MCF~\cite{wang2024swallowing} & 1000 & $3.2 \times 10^{3}$ & 95.0 & 100.0 & 0.103 & \underline{0.044} & 93.7 & 100.0 & 0.119 & 0.055 \\
ET-Flow~\cite{hassan2024flow} & 50 & 76.8 & \textbf{96.5} & 100.0 & \underline{0.073} & 0.047 & \underline{94.1} & 100.0 & \underline{0.098} & \underline{0.039} \\
{OrbDiff}$_\text{ET-Flow}$~\cite{tong2025raoblackwell} & 50 & 76.8 & 96.3 & 100.0 & 0.074 & \textbf{0.027} & 91.9 & 100.0 & 0.113 & 0.042 \\
DiTMC+rPE (B)~\cite{frank2025sampling} & 50 & 35.5 & \underline{96.3}  & 100.0 & \textbf{0.070} & \textbf{0.027} & \textbf{95.7} & 100.0 & \textbf{0.080} & \textbf{0.035} \\
\midrule
GeoMol~\cite{ganea2021geomol} & 1 & 12.5 & 91.5 & 100.0 & 0.225 & 0.193 & 87.6 & 100.0 & 0.270 & 0.241 \\
AvgFlow$_\text{NequIP-D}$~\cite{cao2025efficient} & 1 & 1.8 & 95.1 & 100.0 & 0.220 &0.195 & 84.8 & 100.0 & 0.304 & 0.283 \\
EnFlow-$SO3_\text{with reflow}$~\cite{xu2025energy} & 1 & 6.9 & \textbf{96.7} & 100.0 & 0.122 & 0.087 & \underline{93.5} & 100.0 & 0.170 & 0.132 \\
\textbf{\method$_\text{ET-Flow}$} & 1 & 1.9 & \underline{95.7} & 100.0 & \underline{0.118} & \underline{0.059} & 92.0 & 100.0 & \underline{0.141} & \underline{0.085} \\
\textbf{\method$_\text{DiTMC+airPE (B)}$} & 1 & 1.4 & 95.6 & 100.0 & \textbf{0.106} & \textbf{0.058} & \textbf{93.7} & 100.0 & \textbf{0.124} & \textbf{0.067} \\
\bottomrule
\end{tabular}
\end{adjustbox}
\end{table}

To further evaluate \method, we conduct experiments on the GEOM-DRUGS dataset, which is more challenging due to larger molecular sizes and increased degrees of freedom. Given the richer chemical environments of these molecules, we define the orbits of the embedding through the automorphisms of the molecular graph. As shown in Table~\ref{tab:performance_drugs}, \method achieves state-of-the-art performance across all precision-based metrics among one-step inference baselines.

\begin{table}[ht]
\centering
\caption{Molecule conformer generation results on GEOM-DRUGS ($\delta$= 0.75\AA).  “-R” denotes Recall and “-P” denotes Precision. Best in bold, second-best underlined; results separated into one-step and multi-step. }
\label{tab:performance_drugs}
\begin{adjustbox}{max width=\linewidth}
\begin{tabular}{lcccccccccc}
\toprule
& & Inference & \multicolumn{2}{c}{Coverage-R (\%) $\uparrow$} & \multicolumn{2}{c}{AMR-R (\AA) $\downarrow$} 
& \multicolumn{2}{c}{Coverage-P (\%) $\uparrow$} & \multicolumn{2}{c}{AMR-P  (\AA) $\downarrow$} \\
\cmidrule(lr){4-5} \cmidrule(lr){6-7} \cmidrule(lr){8-9} \cmidrule(lr){10-11}
Method & NFE & time (ms) & Mean & Median & Mean & Median & Mean & Median & Mean & Median \\
\midrule
GeoDiff~\cite{xu2022geodiff} & 5000 & $2.5 \times 10^3$ & 42.1 & 37.8 & 0.835 & 0.809 & 24.9 & 14.5 & 1.136 & 1.090 \\
MCF (L)~\cite{wang2024swallowing} & 1000 & $3.5 \times 10^4$ & \textbf{84.7} & \textbf{92.2} & \textbf{0.390} & \textbf{0.247} & 66.8 & 71.3 & 0.618 & 0.530 \\
ET-Flow~\cite{hassan2024flow} & 50 & 93.7 & \underline{79.6} & \underline{84.6} & 0.439 & 0.406 & \underline{75.2} & \underline{81.7} & \underline{0.517} & \underline{0.442} \\
DiTMC+rPE (L)~\cite{frank2025sampling} & 50 & 27.0 & 78.7 & 84.1 & \underline{0.438} & \underline{0.388} & \textbf{78.1} & \textbf{86.4} & \textbf{0.466} & \textbf{0.381} \\
\midrule
GeoMol~\cite{ganea2021geomol} & 1 & 10.9 & 44.6 & 41.4 & 0.875 & 0.834 & 43.0 & 36.4 & 0.928 & 0.841 \\
AvgFlow$_\text{DiT-D}$~\cite{cao2025efficient} & 1 & 14.6 & \textbf{76.8} & \textbf{82.8} & \textbf{0.548} & \textbf{0.541} & \underline{61.0} & \underline{64.0} & \underline{0.720} & \underline{0.675} \\
EnFlow$_\text{with reflow}$~\cite{xu2025energy} & 1 & 5.8 & \underline{74.3} & \underline{80.0} & \underline{0.566} & \underline{0.548} & {60.6} & {63.5} & \underline{0.719} & \underline{0.666} \\
\textbf{\method$_\text{DiTMC+airPE (L)}$} & 1 & 0.7 & 66.7 & 71.5 & 0.645 & 0.628 & \textbf{63.2} & \textbf{69.2} & \textbf{0.676} & \textbf{0.647} \\
\bottomrule
\end{tabular}
\end{adjustbox}
\end{table}
    
\subsection{Transition state generation}

\paragraph{Datasets and Metrics.}

As an additional task, we consider the generation of TSs. Compared to conformers, which are located around local minima of the potential energy surface, TSs are first-order saddle points and significantly more challenging to model due to their inherent instability. To analyze the performance of \method, we use the RDB7 dataset~\cite{spiekermann2022high}, which contains 11,926 gas-phase reactions involving H, C, N, and O with up to seven heavy atoms. We follow~\citet{kim2024diffusion} and randomly split the dataset into training, validation, and test sets with an 80/10/10 ratio. For each reaction, we generate 25 TS candidates, compute their geometric median, and select as the final prediction the sample closest to this median, following~\citet{galustian2025goflow}. For each prediction, we assess structural accuracy using the RMSD to capture global similarity, as well as the distance matrix absolute error (DMAE) to capture more fine-grained structural differences.

\paragraph{Results.}

Table~\ref{tab:performance_rdb7} reports results on the RDB7 dataset. We adopt the same architecture as GoFlow~\cite{galustian2025goflow} and increase the model size from 5M (B) to 18.7M (L) parameters to account for the increased complexity of the one-step inference setting compared to the multi-step flow-matching baseline. Despite this, the one-shot inference of the larger model is 87$\times$ faster than the original GoFlow architecture with the recommended 25-step inference. Notably, even with a single inference step, the drifting model achieves comparable and in some cases improved accuracy. This efficiency is particularly important for applications involving large reaction networks, where the ability to generate TS candidates at scale is critical, and demonstrates that \method is well-suited for such settings. An additional comparison between one-step flow matching and \method, using the same network architecture, is provided in Section~\ref{sec:add_experiments}.

\begin{table*}[ht]
\small
\centering
\caption{TS generation results on RDB7, with RMSD as a global structural metric and DMAE capturing fine-grained geometric details. Best result in bold across all methods; second-best underlined.}
\vspace{0.5em}
\label{tab:performance_rdb7}
\begin{adjustbox}{max width=\linewidth}
\begin{tabular}{@{\extracolsep\fill}lllllllll}
\toprule
Method & NFE & Inference &\multicolumn{2}{c}{RMSD (\AA)} & \multicolumn{2}{c}{DMAE (\AA)} \\
\cmidrule(lr){4-5} \cmidrule(lr){6-7}
& & time (ms) & Mean & Median & Mean & Median \\
\midrule
TsDiff$_\text{reproduced}$~\cite{kim2024diffusion} & 5000 & 1544 & 0.422 & 0.366 & 0.156 & 0.096 \\
GoFlow$_\text{reproduced}$ (B)~\cite{galustian2025goflow} & 25 & 113 & \textbf{0.328} & \underline{0.258} & \textbf{0.128} & \underline{0.095}  \\
\midrule
TSGen$_\text{reproduced}$~\cite{pattanaik2020generating} & 1 & 9.5 & 0.782 & 0.742 & 0.388 & 0.352 \\
\textbf{\method$_\text{GoFlow (L)}$} & 1 & 1.3 & \textbf{0.328} & \textbf{0.239} & 0.134 & \textbf{0.092} \\
\bottomrule
\end{tabular}
\end{adjustbox}
\end{table*}

\subsection{Different realizations of \method}

Figure~\ref{fig:ablation_drift_realizations} presents an ablation study on the GEOM-QM9 dataset, assessing how different formulations of the drifting field, ranging from base coordinate and iterative alignment to embedding space representations, affect generative performance. Validating Theorem~\ref{thm:sym_problem}, we observe a complete training collapse when applying the base coordinate drift to an equivariant architecture. The model achieves 0\% coverage at the defined threshold of 0.5~\AA\ and produces invalid, overlapping atomic structures. However, as we progressively enforce symmetry within the drifting field, generative quality reliably improves. Although basic rotational alignment prevents complete collapse, the resulting performance is still heavily restricted. Iterative alignment offers further gains but ultimately falls short of the fully $G$-invariant embedding space, largely due to the optimization landscape of the iterative objective trapping the model in local optima. We defer a deeper investigation of these Cartesian realizations and the non-equivariant baseline to Section~\ref{sec:different_realizations}.

\begin{figure}[ht]
\centering
\includegraphics[width=\textwidth]{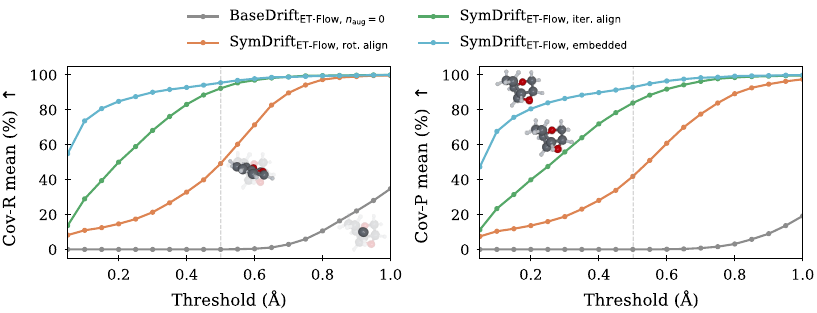}
\caption{
\textbf{Different realizations of \method.} 
We evaluate drifting-fields variants on the GEOM-QM9 dataset using the equivariant ET-Flow architecture, including the base drift and \method with different coordinate-space alignment strategies and embedding-based representations. Example predictions illustrate progressively improved performance with increasing incorporation of symmetry information.
}\label{fig:ablation_drift_realizations}
\end{figure}

\section{Discussion}
\label{sec:discussion}
In this work, we propose \method, a one-shot generative model capable of learning highly symmetrized distributions by explicitly incorporating these symmetries into the design of the drifting field. We demonstrate state-of-the-art performance on conformer and TS generation tasks while significantly reducing inference latency. 

\paragraph{Limitations.} While \method enables highly efficient single-step inference, it increases training time. Our distribution-matching necessitates hierarchically sampling $N_\text{c}$ classes per step, each with $N^{+}$ data samples and $N^{-}$ generated samples. Given fixed memory budgets, this reduces the unique classes per batch compared to standard flow matching, slowing dataset traversal. Additionally, enforcing symmetries introduces overhead, requiring either $N^{+} \times N^{-}$ pairwise alignments per class or computing $G$-invariant embeddings. However, this higher training cost is a favorable trade-off. It is amortized by an order-of-magnitude faster inference, making \method ideal for high-throughput applications like virtual drug screening and large-scale reaction network discovery.

\paragraph{Future Work.} 
A natural limitation of our practical embedding is that it is invariant but not maximally expressive. Constructing invariant representations that distinguish point configurations as finely as possible is a central question in recent work on geometric message-passing networks for 3D point clouds. These works study when invariant or equivariant architectures can distinguish configurations up to Euclidean transformations and permutations, often using geometric variants of the Weisfeiler--Leman test~\citep{joshi2023expressive,hordan2024weisfeiler,sverdlov2025sparse}. This suggests a promising direction for future work: replacing the type-resolved distance multiset used in this paper with more expressive invariant embeddings based on geometric message passing, which may separate molecular configurations more finely while preserving the invariance properties required by the drifting objective.

\begin{ack}

SD would like to thank Johannes Kästner for the guidance and funding provided. This research was funded by the Ministry of Science, Research and the Arts Baden-Wuerttemberg in the Artificial Intelligence Software Academy (AISA). We also acknowledge the support of the Stuttgart Center for Simulation Science (SimTech) and thank the International Max Planck Research School for Intelligent Systems (IMPRS-IS) for support. Tanja was supported by the Deutsche Forschungsgemeinschaft (DFG, German Research Foundation) under Germany’s Excellence Strategy – EXC 2120/1 – 390831618. L.\ Mualem was supported by a postdoctoral scholarship from the Planning and Budgeting Committee (PBC) of the Council for Higher Education in Israel. The authors acknowledge support by the High Performance and Cloud Computing Group at the Zentrum für Datenverarbeitung of the University of Tübingen, the state of Baden-Württemberg through bwHPC and the German Research Foundation (DFG) for funding under "Project number 455787709" (bwForCluster BinAC 2). The authors gratefully acknowledge the computing time provided on the high-performance computer HoreKa by the National High-Performance Computing Center at KIT (NHR@KIT). This center is jointly supported by the Federal Ministry of Education and Research and the Ministry of Science, Research and the Arts of Baden-Württemberg, as part of the National High-Performance Computing (NHR) joint funding program (https://www.nhr-verein.de/en/our-partners). HoreKa is partly funded by the German Research Foundation (DFG). 

\end{ack}

\clearpage


\bibliographystyle{unsrtnat}
\bibliography{literature}

@article{axelrod2022geom,
  title={GEOM, energy-annotated molecular conformations for property prediction and molecular generation},
  author={Axelrod, Simon and Gomez-Bombarelli, Rafael},
  journal={Scientific data},
  volume={9},
  number={1},
  pages={185},
  year={2022},
  publisher={Nature Publishing Group UK London}
}

@article{spiekermann2022high,
  title={High accuracy barrier heights, enthalpies, and rate coefficients for chemical reactions},
  author={Spiekermann, Kevin and Pattanaik, Lagnajit and Green, William H},
  journal={Scientific Data},
  volume={9},
  number={1},
  pages={417},
  year={2022},
  publisher={Nature Publishing Group UK London}
}

@article{deng2026generative,
  title={Generative Modeling via Drifting},
  author={Deng, Mingyang and Li, He and Li, Tianhong and Du, Yilun and He, Kaiming},
  journal={arXiv preprint arXiv:2602.04770},
  year={2026}
}

@inproceedings{
    tong2025learning,
    title={Learning to Discretize Denoising Diffusion {ODE}s},
    author={Vinh Tong and Dung Trung Hoang and Anji Liu and Guy Van den Broeck and Mathias Niepert},
    booktitle={The Thirteenth International Conference on Learning Representations},
    year={2025},
}

@inproceedings{song2023consistency,
  title={Consistency models},
  author={Song, Yang and Dhariwal, Prafulla and Chen, Mark and Sutskever, Ilya},
  booktitle={Proceedings of the 40th International Conference on Machine Learning},
  year={2023}
}

@inproceedings{geng2025mean,
  title={Mean Flows for One-step Generative Modeling},
  author={Geng, Zhengyang and Deng, Mingyang and Bai, Xingjian and Kolter, J Zico and He, Kaiming},
  booktitle={The Thirty-ninth Annual Conference on Neural Information Processing Systems},
  year={2025}
}

@article{chen2024equivariant,
  title={Equivariant score-based generative models provably learn distributions with symmetries efficiently},
  author={Chen, Ziyu and Katsoulakis, Markos A and Zhang, Benjamin J},
  journal={arXiv preprint arXiv:2410.01244},
  year={2024}
}

@inproceedings{kohler2020equivariant,
  title={Equivariant flows: exact likelihood generative learning for symmetric densities},
  author={K{\"o}hler, Jonas and Klein, Leon and No{\'e}, Frank},
  booktitle={International conference on machine learning},
  pages={5361--5370},
  year={2020},
  organization={PMLR}
}

@article{ho2020denoising,
  title={Denoising diffusion probabilistic models},
  author={Ho, Jonathan and Jain, Ajay and Abbeel, Pieter},
  journal={Advances in neural information processing systems},
  volume={33},
  pages={6840--6851},
  year={2020}
}

@inproceedings{
    song2021scorebased,
    title={Score-Based Generative Modeling through Stochastic Differential Equations},
    author={Yang Song and Jascha Sohl-Dickstein and Diederik P Kingma and Abhishek Kumar and Stefano Ermon and Ben Poole},
    booktitle={International Conference on Learning Representations},
    year={2021},
}

@article{karras2022elucidating,
  title={Elucidating the design space of diffusion-based generative models},
  author={Karras, Tero and Aittala, Miika and Aila, Timo and Laine, Samuli},
  journal={Advances in neural information processing systems},
  volume={35},
  pages={26565--26577},
  year={2022}
}

@article{
  tong2024improving,
  title={Improving and generalizing flow-based generative models with minibatch optimal transport},
  author={Alexander Tong and Kilian FATRAS and Nikolay Malkin and Guillaume Huguet and Yanlei Zhang and Jarrid Rector-Brooks and Guy Wolf and Yoshua Bengio},
  journal={Transactions on Machine Learning Research},
  issn={2835-8856},
  year={2024},
}

@inproceedings{
    liu2023flow,
    title={Flow Straight and Fast: Learning to Generate and Transfer Data with Rectified Flow},
    author={Xingchao Liu and Chengyue Gong and qiang liu},
    booktitle={The Eleventh International Conference on Learning Representations },
    year={2023},
}

@inproceedings{
    kim2025simple,
    title={Simple ReFlow: Improved Techniques for Fast Flow Models},
    author={Beomsu Kim and Yu-Guan Hsieh and Michal Klein and marco cuturi and Jong Chul Ye and Bahjat Kawar and James Thornton},
    booktitle={The Thirteenth International Conference on Learning Representations},
    year={2025},
}

@inproceedings{liu2023instaflow,
  title={Instaflow: One step is enough for high-quality diffusion-based text-to-image generation},
  author={Liu, Xingchao and Zhang, Xiwen and Ma, Jianzhu and Peng, Jian and others},
  booktitle={The Twelfth International Conference on Learning Representations},
  year={2023}
}

@inproceedings{
    lee2024improving,
    title={Improving the Training of Rectified Flows},
    author={Sangyun Lee and Zinan Lin and Giulia Fanti},
    booktitle={The Thirty-eighth Annual Conference on Neural Information Processing Systems},
    year={2024},
}

@inproceedings{zhu2024slimflow,
  title={Slimflow: Training smaller one-step diffusion models with rectified flow},
  author={Zhu, Yuanzhi and Liu, Xingchao and Liu, Qiang},
  booktitle={European Conference on Computer Vision},
  pages={342--359},
  year={2024},
  organization={Springer}
}

@InProceedings{hoogeboom2022equidiff,
  title = 	 {Equivariant Diffusion for Molecule Generation in 3{D}},
  author =       {Hoogeboom, Emiel and Satorras, V\'{\i}ctor Garcia and Vignac, Cl{\'e}ment and Welling, Max},
  booktitle = 	 {Proceedings of the 39th International Conference on Machine Learning},
  year = 	 {2022},
  IGNOREpages = 	 {8867--8887},
  IGNOREvolume = 	 {162},
  IGNOREseries = 	 {Proceedings of Machine Learning Research},
  IGNOREmonth = 	 {17--23 Jul},
  IGNOREpublisher =    {PMLR},
}

@article{klein2023equivariant,
  title={Equivariant flow matching},
  author={Klein, Leon and Kr{\"a}mer, Andreas and No{\'e}, Frank},
  journal={Advances in Neural Information Processing Systems},
  IGNOREvolume={36},
  IGNOREpages={59886--59910},
  year={2023}
}

@article{song2023equivariant,
  title={Equivariant flow matching with hybrid probability transport for 3d molecule generation},
  author={Song, Yuxuan and Gong, Jingjing and Xu, Minkai and Cao, Ziyao and Lan, Yanyan and Ermon, Stefano and Zhou, Hao and Ma, Wei-Ying},
  journal={Advances in Neural Information Processing Systems},
  IGNOREvolume={36},
  IGNOREpages={549--568},
  year={2023}
}

@article{zeng2026propmolflow,
  title={PropMolFlow: property-guided molecule generation with geometry-complete flow matching},
  author={Zeng, Cheng and Jin, Jirui and Ambrose, Connor and Karypis, George and Transtrum, Mark and Tadmor, Ellad B and Hennig, Richard G and Roitberg, Adrian and Martiniani, Stefano and Liu, Mingjie},
  journal={Nature Computational Science},
  pages={1--10},
  year={2026},
  publisher={Nature Publishing Group US New York}
}

@inproceedings{lipman2023flow,
title={Flow Matching for Generative Modeling},
author={Yaron Lipman and Ricky T. Q. Chen and Heli Ben-Hamu and Maximilian Nickel and Matthew Le},
booktitle={The Eleventh International Conference on Learning Representations },
year={2023}
}

@article{albergo2025stochastic,
  title={Stochastic interpolants: A unifying framework for flows and diffusions},
  author={Albergo, Michael and Boffi, Nicholas M and Vanden-Eijnden, Eric},
  journal={Journal of Machine Learning Research},
  volume={26},
  number={209},
  pages={1--80},
  year={2025}
}

@inproceedings{
    zheng2026diffusion,
    title={Diffusion Transformers with Representation Autoencoders},
    author={Boyang Zheng and Nanye Ma and Shengbang Tong and Saining Xie},
    booktitle={The Fourteenth International Conference on Learning Representations},
    year={2026},
}

@article{pracht2024crest,
  title={CREST—A program for the exploration of low-energy molecular chemical space},
  author={Pracht, Philipp and Grimme, Stefan and Bannwarth, Christoph and Bohle, Fabian and Ehlert, Sebastian and Feldmann, Gereon and Gorges, Johannes and M{\"u}ller, Marcel and Neudecker, Tim and Plett, Christoph and others},
  journal={The Journal of Chemical Physics},
  volume={160},
  number={11},
  year={2024},
  publisher={AIP Publishing}
}

@article{riniker2015better,
  title={Better informed distance geometry: using what we know to improve conformation generation},
  author={Riniker, Sereina and Landrum, Gregory A},
  journal={Journal of chemical information and modeling},
  volume={55},
  number={12},
  pages={2562--2574},
  year={2015},
  publisher={ACS Publications}
}

@article{luo2021predicting,
  title={Predicting molecular conformation via dynamic graph score matching},
  author={Luo, Shitong and Shi, Chence and Xu, Minkai and Tang, Jian},
  journal={Advances in neural information processing systems},
  volume={34},
  pages={19784--19795},
  year={2021}
}

@inproceedings{shi2021learning,
  title={Learning gradient fields for molecular conformation generation},
  author={Shi, Chence and Luo, Shitong and Xu, Minkai and Tang, Jian},
  booktitle={International conference on machine learning},
  pages={9558--9568},
  year={2021},
  organization={PMLR}
}

@article{xu2022geodiff,
  title={Geodiff: A geometric diffusion model for molecular conformation generation},
  author={Xu, Minkai and Yu, Lantao and Song, Yang and Shi, Chence and Ermon, Stefano and Tang, Jian},
  journal={arXiv preprint arXiv:2203.02923},
  year={2022}
}

@inproceedings{
    ganea2021geomol,
    title={GeoMol: Torsional Geometric Generation of Molecular 3D Conformer Ensembles},
    author={Octavian-Eugen Ganea and Lagnajit Pattanaik and Connor W. Coley and Regina Barzilay and Klavs Jensen and William Green and Tommi S. Jaakkola},
    booktitle={Advances in Neural Information Processing Systems},
    editor={A. Beygelzimer and Y. Dauphin and P. Liang and J. Wortman Vaughan},
    year={2021},
}

@article{jing2022torsional,
  title={Torsional diffusion for molecular conformer generation},
  author={Jing, Bowen and Corso, Gabriele and Chang, Jeffrey and Barzilay, Regina and Jaakkola, Tommi},
  journal={Advances in neural information processing systems},
  volume={35},
  pages={24240--24253},
  year={2022}
}

@inproceedings{wang2024swallowing,
  title={Swallowing the Bitter Pill: Simplified Scalable Conformer Generation},
  author={Wang, Yuyang and Elhag, Ahmed AA and Jaitly, Navdeep and Susskind, Joshua M and Bautista, Miguel {\'A}ngel},
  booktitle={International Conference on Machine Learning},
  pages={50400--50418},
  year={2024},
  organization={PMLR}
}

@article{hassan2024flow,
  title={Et-flow: Equivariant flow-matching for molecular conformer generation},
  author={Hassan, Majdi and Shenoy, Nikhil and Lee, Jungyoon and St{\"a}rk, Hannes and Thaler, Stephan and Beaini, Dominique},
  journal={Advances in Neural Information Processing Systems},
  volume={37},
  pages={128798--128824},
  year={2024}
}

@inproceedings{
liu2025nextmol,
    title={{NE}xT-Mol: 3D Diffusion Meets 1D Language Modeling for 3D Molecule Generation},
    author={Zhiyuan Liu and Yanchen Luo and Han Huang and Enzhi Zhang and Sihang Li and Junfeng Fang and Yaorui Shi and Xiang Wang and Kenji Kawaguchi and Tat-Seng Chua},
    booktitle={The Thirteenth International Conference on Learning Representations},
    year={2025},
}

@article{frank2025sampling,
  title={Sampling 3D Molecular Conformers with Diffusion Transformers},
  author={Frank, J Thorben and Ripken, Winfried and Lied, Gregor and M{\~A}{\v{z}}ller, Klaus-Robert and Unke, Oliver T and Chmiela, Stefan},
  journal={arXiv preprint arXiv:2506.15378},
  year={2025}
}

@inproceedings{
    cao2025efficient,
    title={Efficient Molecular Conformer Generation with {SO}(3)-Averaged Flow Matching and Reflow},
    author={Zhonglin Cao and Mario Geiger and Allan Dos Santos Costa and Danny Reidenbach and Karsten Kreis and Tomas Geffner and Franco Pellegrini and Guoqing Zhou and Emine Kucukbenli},
    booktitle={Forty-second International Conference on Machine Learning},
    year={2025},
}

@article{xu2025energy,
  title={Energy-Guided Flow Matching Enables Few-Step Conformer Generation and Ground-State Identification},
  author={Xu, Guikun and Yi, Xiaohan and Zhao, Peilin and Bian, Yatao},
  journal={arXiv preprint arXiv:2512.22597},
  year={2025}
}

@inproceedings{tong2025raoblackwell,
  title     = {Rao-Blackwell Gradient Estimators for Equivariant Denoising Diffusion},
  author    = {Tong, Vinh and Hoang, Trung-Dung and Liu, Anji and Van den Broeck, Guy and Niepert, Mathias},
  booktitle = {Proceedings of the 39th Conference on Neural Information Processing Systems},
  year      = {2025}
}

@article{henkelman_climbing_2000,
	title = {A climbing image nudged elastic band method for finding saddle points and minimum energy paths},
	volume = {113},
	issn = {0021-9606},
	doi_ = {10.1063/1.1329672},
	number = {22},
	urldate = {2024-06-24},
	journal = {J. Chem. Phys.},
	author = {Henkelman, Graeme and Uberuaga, Blas P. and Jónsson, Hannes},
	month = dec,
	year = {2000},
	pages = {9901--9904},
}

@incollection{jonsson1998nudged,
  title={Nudged elastic band method for finding minimum energy paths of transitions},
  author={J{\'o}nsson, Hannes and Mills, Greg and Jacobsen, Karsten W},
  booktitle={Classical and quantum dynamics in condensed phase simulations},
  pages={385--404},
  year={1998},
  publisher={World Scientific}
}

@article{banerjee_search_1985,
  title={Search for stationary points on surfaces},
  author={Banerjee, Ajit and Adams, Noah and Simons, Jack and Shepard, Ron},
  journal={The Journal of Physical Chemistry},
  volume={89},
  number={1},
  pages={52--57},
  year={1985},
  publisher={ACS Publications}
}

@article{henkelman_dimer_1999,
	title = {A dimer method for finding saddle points on high dimensional potential surfaces using only first derivatives},
	volume = {111},
	issn = {0021-9606},
	doi_ = {10.1063/1.480097},
	number = {15},
	urldate = {2024-06-24},
	journal = {J. Chem. Phys.},
	author = {Henkelman, Graeme and Jónsson, Hannes},
	month = oct,
	year = {1999},
	pages = {7010--7022},
}

@article{wander_cattsunami_2024,
    author = {Wander, Brook and Shuaibi, Muhammed and Kitchin, John R. and Ulissi, Zachary W. and Zitnick, C. Lawrence},
    title = {CatTSunami: Accelerating Transition State Energy Calculations with Pretrained Graph Neural Networks},
    journal = {ACS Catal.},
    volume = {15},
    number = {7},
    pages = {5283-5294},
    year = {2025},
}

@article{yuan_analytical_2024,
	title = {Analytical ab initio hessian from a deep learning potential for transition state optimization},
	volume = {15},
	copyright = {2024 The Author(s)},
	issn = {2041-1723},
	doi_ = {10.1038/s41467-024-52481-5},
	language = {en},
	number = {1},
	urldate = {2024-12-14},
	journal = {Nat. Commun.},
	author = {Yuan, Eric C.-Y. and Kumar, Anup and Guan, Xingyi and Hermes, Eric D. and Rosen, Andrew S. and Zádor, Judit and Head-Gordon, Teresa and Blau, Samuel M.},
	month = oct,
	year = {2024},
	pages = {8865},
}

@article{galustian2026motsart,
  title={moTSart: Accelerating Automated Transition State Search with Generative Models in a Low-Data Regime},
  author={Galustian, Leonard and Karwounopoulos, Johannes and Demuth, Tori and De Landsheere, Jasper and Mark, Konstantin and Kovar, Maximilian P-P and Zamyatin, Anton and Svatunek, Dennis and Heid, Esther},
  journal = {ChemRxiv},
  year={2026},
  volume = {2026},
  number = {0417},
}

@article{kim2024diffusion,
  title={Diffusion-based generative AI for exploring transition states from 2D molecular graphs},
  author={Kim, Seonghwan and Woo, Jeheon and Kim, Woo Youn},
  journal={Nature Communications},
  volume={15},
  number={1},
  pages={341},
  year={2024},
  publisher={Nature Publishing Group UK London}
}

@article{galustian2025goflow,
  title={GoFlow: efficient transition state geometry prediction with flow matching and E (3)-equivariant neural networks},
  author={Galustian, Leonard and Mark, Konstantin and Karwounopoulos, Johannes and Kovar, Maximilian P-P and Heid, Esther},
  journal={Digital Discovery},
  volume={4},
  number={12},
  pages={3492--3501},
  year={2025},
  publisher={Royal Society of Chemistry}
}

@article{duan2023accurate,
  title={Accurate transition state generation with an object-aware equivariant elementary reaction diffusion model},
  author={Duan, Chenru and Du, Yuanqi and Jia, Haojun and Kulik, Heather J},
  journal={Nature computational science},
  volume={3},
  number={12},
  pages={1045--1055},
  year={2023},
  publisher={Nature Publishing Group US New York}
}

@article{darouich2026adaptive,
  title={Adaptive Transition-State Refinement with Learned Equilibrium Flows},
  author={Darouich, Samir and Tong, Vinh and Bien, Tanja and Kästner, Johannes and Niepert, Mathias},
  journal={Journal of Chemical Information and Modeling},
  volume={66},
  number={4},
  pages={2154--2165},
  year={2026},
  publisher={ACS Publications}
}

@article{pattanaik2020generating,
  title={Generating transition states of isomerization reactions with deep learning},
  author={Pattanaik, Lagnajit and Ingraham, John B and Grambow, Colin A and Green, William H},
  journal={Physical Chemistry Chemical Physics},
  volume={22},
  number={41},
  pages={23618--23626},
  year={2020},
  publisher={Royal Society of Chemistry}
}

@article{heid2021machine,
  title={Machine learning of reaction properties via learned representations of the condensed graph of reaction},
  author={Heid, Esther and Green, William H},
  journal={Journal of Chemical Information and Modeling},
  volume={62},
  number={9},
  pages={2101--2110},
  year={2021},
  publisher={ACS Publications}
}

@inproceedings{joshi2023expressive,
  title={On the Expressive Power of Geometric Graph Neural Networks},
  author={Joshi, Chaitanya K. and Bodnar, Cristian and Mathis, Simon V. and Cohen, Taco and Li{\`o}, Pietro},
  booktitle={International Conference on Machine Learning},
  year={2023}
}

@inproceedings{hordan2024weisfeiler,
  title={Weisfeiler Leman for Euclidean Equivariant Machine Learning},
  author={Hordan, Snir and Amir, Tal and Dym, Nadav},
  booktitle={International Conference on Machine Learning},
  year={2024}
}

@inproceedings{sverdlov2025sparse,
  title={On the Expressive Power of Sparse Geometric MPNNs},
  author={Sverdlov, Yonatan and Dym, Nadav},
  booktitle={International Conference on Learning Representations},
  year={2025}
}

@article{kabsch1976solution,
  title={A solution for the best rotation to relate two sets of vectors},
  author={Kabsch, Wolfgang},
  journal={Foundations of Crystallography},
  volume={32},
  number={5},
  pages={922--923},
  year={1976},
  publisher={International Union of Crystallography}
}

@article{kuhn1955hungarian,
  title={The Hungarian method for the assignment problem},
  author={Kuhn, Harold W},
  journal={Naval research logistics quarterly},
  volume={2},
  number={1-2},
  pages={83--97},
  year={1955},
  publisher={Wiley Online Library}
}

@article{abramson2024accurate,
  title={Accurate structure prediction of biomolecular interactions with AlphaFold 3},
  author={Abramson, Josh and Adler, Jonas and Dunger, Jack and Evans, Richard and Green, Tim and Pritzel, Alexander and Ronneberger, Olaf and Willmore, Lindsay and Ballard, Andrew J and Bambrick, Joshua and others},
  journal={Nature},
  volume={630},
  number={8016},
  pages={493--500},
  year={2024},
  publisher={Nature Publishing Group UK London}
}

@article{geffner2025proteina,
  title={Proteina: Scaling flow-based protein structure generative models},
  author={Geffner, Tomas and Didi, Kieran and Zhang, Zuobai and Reidenbach, Danny and Cao, Zhonglin and Yim, Jason and Geiger, Mario and Dallago, Christian and Kucukbenli, Emine and Vahdat, Arash and others},
  journal={arXiv preprint arXiv:2503.00710},
  year={2025}
}







\appendix




\newpage

\begin{center}
    \hrule 
    \startcontents[sections]\vbox{\vspace{4mm}\sc \LARGE SymDrift: One-Shot Generative Modeling under Symmetries \\\sc\small \textbf{Additional Material}} \vspace{5mm} \hrule height .5pt
    \printcontents[sections]{l}{0}{\setcounter{tocdepth}{2}}
\end{center}

\newpage

\section{Extented Background}
\label{sec:extended_background}

\paragraph{Invariance and Equivariance.}  
A function $f: \Omega \to \mathbb{R}$ is $\mathcal{G}$-invariant if its output does not change under any group action:
\begin{equation}
    f(g \cdot x) = f(x), \quad \forall g \in \mathcal{G}, x \in \Omega.
\end{equation}
A function $f: \Omega \to \Omega$ is $\mathcal{G}$-equivariant if it commutes with the group action:
\begin{equation}
    f(g \cdot x) = g \cdot f(x), \quad \forall g \in \mathcal{G}, x \in \Omega.
\end{equation}
Equivariance ensures that transformations of the input propagate consistently to the output. A probability distribution $p(x)$ on $\Omega$ is $\mathcal{G}$-invariant if it is unchanged under any group action:
\begin{equation}
p(g \cdot x \in A) = p(x \in A), \quad \forall g \in \mathcal{G}, \text{ measurable } A \subseteq \Omega.
\end{equation}
Invariant distributions are essential for designing models that respect the symmetries of the underlying data.

\textbf{Group Symmetrization.}

Let $\mathcal{S}_G$ denote the symmetrization operator associated with a group $G$, which maps any distribution $p(\mathbf{x})$ to a $G$-invariant distribution, referred to as the $G$-symmetrized distribution:
\begin{equation}
    \mathcal{S}_G[p](\mathbf{x})
    :=
    \int_G p(g \cdot \mathbf{x}) \, \mathrm{d}\mu_G(g)
\end{equation}
where $\mu_G$ is the Haar measure on $G$. The Haar measure is a measure defined on a locally compact topological group $G$ that is invariant under group translations, i.e., $\mu_G(gA) = \mu_G(A)$ for all measurable sets $A \subseteq G$ and all $g \in G$. It is unique up to a constant scaling factor.

\section{Theoretical Proofs}
\label{sec:proofs}

\subsection{Gradient definition}
\label{sec:gradient_def}

We begin with the definition of the drifting loss:
\begin{equation}
    \mathcal{L} = \mathbb{E}_{{\boldsymbol{\epsilon}} \sim  p_\epsilon} \left[ \ell(f_{\theta}(\boldsymbol\epsilon)) \right]
\end{equation}
where the point-wise loss function is defined as:
\begin{equation}
    \ell(\mathbf{x}) = \left\| \mathbf{x} - \text{stopgrad} \left( \mathbf{x} + \mathbf{V}_{p,q_{\theta}}(\mathbf{x}) \right) \right\|^{2}
\end{equation}
with $\mathbf{x} = f_\theta(\boldsymbol{\epsilon})$ as the model prediction. By the multivariate chain rule, the gradient of the loss with respect to the parameters $\theta$ is the product of the Jacobian of the generator and the gradient of $\ell$ with respect to the prediction:
\begin{equation}
    \nabla_\theta\mathcal{L} = \mathbb{E}_{{\boldsymbol{\epsilon}}} \left[ \left( \nabla_\theta f_\theta(\boldsymbol{\epsilon}) \right)^\top \nabla_{\mathbf{x}} \ell(\mathbf{x}) \right]
\end{equation}
To compute $\nabla_{\mathbf{x}} \ell(\mathbf{x})$, we denote the target term as $\mathbf{y}_{\text{target}} = \text{sg} (\mathbf{x} + \mathbf{V}_{p,q_{\theta}}(\mathbf{x}))$. Since the \texttt{sg} operator ensures $\nabla_{\mathbf{x}} \mathbf{y}_{\text{target}} = \mathbf{0}$, we have:
\begin{align}
    \nabla_{\mathbf{x}} \ell(\mathbf{x}) &= \nabla_{\mathbf{x}} \left\| \mathbf{x} - \mathbf{y}_{\text{target}} \right\|^{2} \\
    &= 2 (\mathbf{x} - \mathbf{y}_{\text{target}})
\end{align}
Evaluating the term inside the parentheses using the forward-pass value of $\mathbf{y}_{\text{target}}$:
\begin{align}
    \mathbf{x} - \mathbf{y}_{\text{target}} &= \mathbf{x} - (\mathbf{x} + \mathbf{V}_{p,q_{\theta}}(\mathbf{x})) \\
    &= - \mathbf{V}_{p,q_{\theta}}(\mathbf{x})
\end{align}
Substituting $\nabla_{\mathbf{x}} \ell(\mathbf{x}) = -2 \mathbf{V}_{p,q_{\theta}}(\mathbf{x})$ and $\mathbf{V}_{p,q_{\theta}}(\mathbf{x}) = \mathbf{V}_{p}^{+}(\mathbf{x}) - \mathbf{V}_{q_\theta}^{-}(\mathbf{x})$ back into the chain rule expression, we obtain the final gradient estimator:
\begin{equation}
    \label{eq:loss_gradient}
    \nabla_\theta\mathcal{L} = \mathbb{E}_{{\boldsymbol{\epsilon}} \sim  p_\epsilon} \left[ -2\left(\nabla_{\theta} f_\theta(\boldsymbol{\epsilon}) \right)^\top \left( \mathbf{V}_{p}^{+}(\mathbf{x}) - \mathbf{V}_{q_\theta}^{-}(\mathbf{x}) \right) \right]
\end{equation}

In the following, we derive the gradient from Equation~\ref{eq:conflict_grad}:
\begin{align}
    \label{eq:g_dependent_grad}
    \nabla_\theta\mathcal{L} &= \mathbb{E}_{{\boldsymbol{\epsilon}} \sim  p_\epsilon} \left[ -2\left(\nabla_{\theta} f_\theta(\boldsymbol{\epsilon}) \right)^\top \left( \mathbf{V}_{p}^{+}(\mathbf{x}) - \mathbf{V}_{q_\theta}^{-}(\mathbf{x}) \right) \right] \\
    &= \mathbb{E}_{\boldsymbol{\epsilon} \sim p_{\epsilon}} \mathbb{E}_{g \sim G} \left[ -2 \bigl(\nabla_\theta f_\theta(g\boldsymbol{\epsilon})\bigr)^\top \bigl(\mathbf{V}^+_p(g\mathbf{x}) - \mathbf{V}^-_{q_\theta}(g\mathbf{x})\bigr) \right] \quad G\text{-invariance of } p_\epsilon \, (p_\epsilon(\boldsymbol{\epsilon}) = p_\epsilon(g\boldsymbol{\epsilon})) \\
    &= \mathbb{E}_{\boldsymbol{\epsilon} \sim p_{\epsilon}} \mathbb{E}_{g \sim G} \left[ -2 \bigl(g\nabla_\theta f_\theta(\boldsymbol{\epsilon})\bigr)^\top \bigl(\mathbf{V}^+_p(g\mathbf{x}) - g\mathbf{V}^-_{q_\theta}(\mathbf{x})\bigr) \right] \quad \text{equvariance of } f_\theta \, (f_\theta(g\boldsymbol{\epsilon})=gf_\theta(\boldsymbol{\epsilon}))\\
    &= \mathbb{E}_{\boldsymbol{\epsilon} \sim p_{\epsilon}} \mathbb{E}_{g \sim G} \left[ -2 \bigl(\nabla_\theta f_\theta(\boldsymbol{\epsilon})\bigr)^\top g^{-1}\bigl(\mathbf{V}^+_p(g\mathbf{x}) - g\mathbf{V}^-_{q_\theta}(\mathbf{x})\bigr) \right] \quad \text{orthogonality of } g \, (g^\top = g^{-1})\\
    &=\mathbb{E}_{\boldsymbol{\epsilon} \sim p_{\epsilon}} \Big[
    -2\, \bigl(\nabla_\theta f_\theta(\boldsymbol{\epsilon})\bigr)^\top
    \big(\mathbb{E}_{g \sim G} \bigl[g^{-1} \mathbf{V}^+_p(g\mathbf{x})\bigr]
    - \mathbf{V}^-_{q_\theta}(\mathbf{x})\big)
    \Big] \quad g \text{ independence of } \mathbf{V}^{-}_{q_\theta} (g^{-1}g=I)
\end{align}

\subsection{Main Theorem}
\label{appendix: main_thm}

\begingroup
\setcounter{theorem}{0} 
\begin{tcolorbox}[title=Drift field mismatch under symmetrization]
\begin{theorem}
Let \(G\) be a compact orthogonal group on \(\mathbb{R}^d\) with Haar measure, and define
\[
\mathbf{V}_p^+(\mathbf{x})
=
\frac{\mathbb{E}_{\mathbf{y}^+ \sim p}\!\left[k(\mathbf{x},\mathbf{y}^+)(\mathbf{y}^+-\mathbf{x})\right]}
     {\mathbb{E}_{\mathbf{y}^+ \sim p}\!\left[k(\mathbf{x},\mathbf{y}^+)\right]},
\]
where 
\(
k(\mathbf{x},\mathbf{y}^+)=\exp\!\left(-\frac{\|\mathbf{x}-\mathbf{y}^+\|}{\tau}\right)
\)
is the Laplacian kernel.

Let \(p^G(\mathbf{x})=\mathbb{E}_{g\sim G}[p(g^{-1}\mathbf{x})]\) and
\(
\hat{\mathbf{V}}_{p}^+(\mathbf{x})=\mathbb{E}_{g\sim G}\!\left[g^{-1}\mathbf{V}_p^+(g\mathbf{x})\right].
\)
Then there exist a non-\(G\)-invariant \(p\) and \(\mathbf{x} \in \mathbb{R}^d\) such that
\(
\hat{\mathbf{V}}_{p}^+(\mathbf{x}) \neq \mathbf{V}_{p^G}^+(\mathbf{x}).
\)
\end{theorem}
\end{tcolorbox}

\begin{proof}
We derive explicit expressions for both sides. We divide the proof into three steps:
\begin{enumerate}
    \item Expression for $\mathbf{V}_{p^G}^+(\mathbf{x})$.
    \item Expression for $\hat{\mathbf{V}}_{p}^+(\mathbf{x})$.
    \item Comparison.
\end{enumerate}

\paragraph{Step 1: Symmetrized distribution drift.}

By definition,
\[
\mathbf{V}_{p^G}^+(\mathbf{x})
=
\frac{
\mathbb{E}_{\mathbf{y}^+ \sim p^G}
\bigl[
k(\mathbf{x},\mathbf{y}^+)(\mathbf{y}^+-\mathbf{x})
\bigr]
}{
\mathbb{E}_{\mathbf{y}^+ \sim p^G}
\bigl[
k(\mathbf{x},\mathbf{y}^+)
\bigr]
}.
\]

Using $p^G(\mathbf{y}) = \mathbb{E}_{g \sim G}[p(g^{-1}\mathbf{y})]$, we rewrite the numerator:
\[
A_{p^G}(\mathbf{x})
=
\mathbb{E}_{\mathbf{y}^+ \sim p}\mathbb{E}_{g \sim G}
\bigl[
k(\mathbf{x}, g\mathbf{y}^+)(g\mathbf{y}^+-\mathbf{x})
\bigr] 
= 
\mathbb{E}_{g \sim G} \mathbb{E}_{\mathbf{y}^+ \sim p}
\bigl[
k(\mathbf{x}, g\mathbf{y}^+)(g\mathbf{y}^+-\mathbf{x})
\bigr]
\]

Using orthogonality of $G$, we have
\(
k(\mathbf{x}, g\mathbf{y}) = k(g^{-1}\mathbf{x}, \mathbf{y}),
\)
and we rewrite
\(
g\mathbf{y} - \mathbf{x}
=
g(\mathbf{y} - g^{-1}\mathbf{x}).
\)

Hence,
\[
A_{p^G}(\mathbf{x})
=
\mathbb{E}_{g \sim G}
\Big[
g \,
\mathbb{E}_{\mathbf{y} \sim p}
\bigl[
k(g^{-1}\mathbf{x}, \mathbf{y})(\mathbf{y}-g^{-1}\mathbf{x})
\bigr]
\Big]
=
\mathbb{E}_{g \sim G}
\bigl[
g\, A_p(g^{-1}\mathbf{x})
\bigr].
\]

Similarly,
\[
Z_{p^G}(\mathbf{x})
=
\mathbb{E}_{g \sim G}
\bigl[
Z_p(g^{-1}\mathbf{x})
\bigr].
\]

Thus,
\[
\mathbf{V}_{p^G}^+(\mathbf{x})
=
\frac{
\mathbb{E}_{g \sim G}
\bigl[
g\, A_p(g^{-1}\mathbf{x})
\bigr]
}{
\mathbb{E}_{g \sim G}
\bigl[
Z_p(g^{-1}\mathbf{x})
\bigr]
}.
\]

\paragraph{Step 2: Aggregated drifting field.}

By definition,
\[
\hat{\mathbf{V}}_{p}^+(\mathbf{x})
=
\mathbb{E}_{g \sim G}
\bigl[
g^{-1}\mathbf{V}_p^+(g\mathbf{x})
\bigr].
\]

Using $\mathbf{V}_p^+(\mathbf{x}) = \frac{A_p(\mathbf{x})}{Z_p(\mathbf{x})}$, we obtain
\[
\hat{\mathbf{V}}_{p}^+(\mathbf{x})
=
\mathbb{E}_{g \sim G}
\left[
g^{-1}
\frac{A_p(g\mathbf{x})}{Z_p(g\mathbf{x})}
\right]
=
\mathbb{E}_{g \sim G}
\left[
g
\frac{A_p(g^{-1}\mathbf{x})}{Z_p(g^{-1}\mathbf{x})}
\right].
\]

\paragraph{Step 3: Comparison.}

We compare
\[
\hat{\mathbf{V}}_{p}^+(\mathbf{x})
=
\mathbb{E}_{g \sim G}
\left[
g
\frac{A_p(g^{-1}\mathbf{x})}{Z_p(g^{-1}\mathbf{x})}
\right],
\]
and
\[
\mathbf{V}_{p^G}^+(\mathbf{x})
=
\frac{
\mathbb{E}_{g \sim G}
\bigl[
g\, A_p(g^{-1}\mathbf{x})
\bigr]
}{
\mathbb{E}_{g \sim G}
\bigl[
Z_p(g^{-1}\mathbf{x})
\bigr]
}.
\]

The difference is structural:
\begin{itemize}
\item $\hat{\mathbf{V}}$ averages the ratio $\frac{A_p}{Z_p}$,
\item $\mathbf{V}_{p^G}^+$ is the ratio of averaged numerator and denominator.
\end{itemize}

Since $(A,Z)\mapsto A/Z$ is nonlinear, these operations do not commute. Therefore, in general,
\[
\hat{\mathbf{V}}_{p}^+(\mathbf{x})
\neq
\mathbf{V}_{p^G}^+(\mathbf{x}).
\]

\paragraph{Example (vector-valued Laplace kernel on $\mathrm{SO}(3)$).}
Let $G = \mathrm{SO}(3)$ act on $\mathbb{R}^{N \times 3}$ by
$(g\mathbf{x})_i = g \mathbf{x}_i$ for $i=1,\dots,N$, and let
$\mathbf{y}^+ \in \mathbb{R}^{N \times 3}$ be a zero-centered structure.
Define
\[
A_p(\mathbf{x}) = \exp\bigl(-\|\mathbf{x} - \mathbf{y}^+\| / \tau\bigr)\,(\mathbf{y}^+ - \mathbf{x}),
\quad
Z_p(\mathbf{x}) = \exp\bigl(-\|\mathbf{x} - \mathbf{y}^+\| / \tau \bigr).
\]

Then we have
\[
\frac{A_p(\mathbf{x})}{Z_p(\mathbf{x})} = \frac{\exp\bigl(-\|\mathbf{x} - \mathbf{y}^+\| / \tau\bigr)\,(\mathbf{y}^+ - \mathbf{x})}{\exp\bigl(-\|\mathbf{x} - \mathbf{y}^+\| / \tau \bigr)} = (\mathbf{y}^+ - \mathbf{x}).
\]

Thus,
\[
\hat{\mathbf{V}}_{p}^+(\mathbf{x})
=
\mathbb{E}_{g \sim G}
\left[
g \frac{A_p(g^{-1}\mathbf{x})}{Z_p(g^{-1}\mathbf{x})}
\right]
=
\mathbb{E}_{g \sim G}
\bigl[
g(\mathbf{y}^+ - g^{-1}\mathbf{x})
\bigr] = -\mathbf{x}.
\]
(Note: This final equality holds because the integral of the coordinate functions over the Haar measure of $\mathrm{SO}(3)$ vanishes, meaning $\mathbb{E}_{g \sim G}[g\mathbf{y}^+] = \mathbf{0}$.)

On the other hand,
\[
\mathbf{V}_{p^G}^+(\mathbf{x})
=
\frac{
\mathbb{E}_{g}
\bigl[
g\, A_p(g^{-1}\mathbf{x})
\bigr]
}{
\mathbb{E}_{g}
\bigl[
Z_p(g^{-1}\mathbf{x})
\bigr]
}.
\]

We observe that $\hat{\mathbf{V}}_{p}^+(\mathbf{x}) = -\mathbf{x}$ is independent of both
$\mathbf{y}^+$ and the temperature $\tau$.

We now analyze $\mathbf{V}_{p^G}^+(\mathbf{x})$ in the limit $\tau \to 0$.
Define
\[
w_\tau(g) := \exp\bigl(-\|g^{-1}\mathbf{x} - \mathbf{y}^+\|/\tau\bigr).
\]
Then
\[
\mathbf{V}_{p^G}^+(\mathbf{x})
=
\frac{
\mathbb{E}_{g}\bigl[w_\tau(g)\, g(\mathbf{y}^+ - g^{-1}\mathbf{x})\bigr]
}{
\mathbb{E}_{g}\bigl[w_\tau(g)\bigr]
}.
\]

As $\tau \to 0$, the weights $w_\tau(g)$ concentrate on the set
\[
\mathcal{G}^* := \arg\min_{g \in G} \|g^{-1}\mathbf{x} - \mathbf{y}^+\|.
\]
Assume for simplicity that the minimizer is unique, i.e.,
\[
g^* = \arg\min_{g \in G} \|g^{-1}\mathbf{x} - \mathbf{y}^+\|.
\]
Then $w_\tau(g)$ concentrates at $g^*$, and we obtain
\[
\lim_{\tau \to 0} \mathbf{V}_{p^G}^+(\mathbf{x})
=
g^* \bigl(\mathbf{y}^+ - (g^*)^{-1}\mathbf{x}\bigr)
=
g^*\mathbf{y}^+ - \mathbf{x}.
\]

In general, $g^*\mathbf{y}^+ \neq \mathbf{0}$, hence
\[
\lim_{\tau \to 0} \mathbf{V}_{p^G}^+(\mathbf{x})
\neq -\mathbf{x}
=
\hat{\mathbf{V}}_{p}^+(\mathbf{x}).
\]

Therefore, $\hat{\mathbf{V}}_{p}^+(\mathbf{x}) \neq \mathbf{V}_{p^G}^+(\mathbf{x})$
in general.

This concludes the proof.

\end{proof}

\subsection{Equivariant Drifting Fields from Invariant Distributions}

\label{proof_property1}

\begin{tcolorbox}[title=Equivariance of Kernel Drift under Invariant Distributions]
\textbf{Property 1.}
Let \(G\) be a compact group acting orthogonally on \(\mathbb{R}^d\) with
normalized Haar measure, and assume the kernel satisfies
\(k(g\mathbf{x}, g\mathbf{y}) = k(\mathbf{x}, \mathbf{y})\).
If a distribution \(q\) is \(G\)-invariant, then the kernel drift satisfies
\[
\mathbf{V}^{-}_{q}(g\mathbf{x}) = g\,\mathbf{V}^{-}_{q}(\mathbf{x}),
\qquad \forall g \in G, \mathbf{x} \in \mathbb{R}^d.
\]
\end{tcolorbox}

\begin{proof}
We start from the definition of \(\mathbf{V}^{-}_{q}(g\mathbf{x})\):
\[
\mathbf{V}^{-}_{q}(g\mathbf{x})
=
\frac{1}{Z_q(g\mathbf{x})}
\mathbb{E}_{\mathbf{y}^- \sim q}
\left[
k(g\mathbf{x}, \mathbf{y}^-)\,(\mathbf{y}^- - g\mathbf{x})
\right].
\]

Since \(q\) is \(G\)-invariant, we can write \(\mathbf{y}^- = g\mathbf{z}^-\) with \(\mathbf{z}^- \sim q\), giving
\[
\mathbf{V}^{-}_{q}(g\mathbf{x})
=
\frac{1}{Z_q(g\mathbf{x})}
\mathbb{E}_{\mathbf{z}^- \sim q}
\left[
k(g\mathbf{x}, g\mathbf{z}^-)\,(g\mathbf{z}^- - g\mathbf{x})
\right].
\]

Using kernel invariance \(k(g\mathbf{x}, g\mathbf{z}^-) = k(\mathbf{x}, \mathbf{z}^-)\) and linearity of the group action \(g\mathbf{z}^- - g\mathbf{x} = g(\mathbf{z}^- - \mathbf{x})\), we obtain
\[
\mathbf{V}^{-}_{q}(g\mathbf{x})
=
\frac{1}{Z_q(g\mathbf{x})}
\mathbb{E}_{\mathbf{z}^- \sim q}
\left[
k(\mathbf{x}, \mathbf{z}^-)\, g(\mathbf{z}^- - \mathbf{x})
\right].
\]

Since \(g\) is linear, it can be factored out of the expectation:
\[
\mathbf{V}^{-}_{q}(g\mathbf{x})
=
g \cdot
\frac{1}{Z_q(g\mathbf{x})}
\mathbb{E}_{\mathbf{z}^- \sim q}
\left[
k(\mathbf{x}, \mathbf{z}^-)\,(\mathbf{z}^- - \mathbf{x})
\right].
\]

Finally, note that \(Z_q(g\mathbf{x}) = Z_q(\mathbf{x})\) by the same change of variables, hence
\[
\mathbf{V}^{-}_{q}(g\mathbf{x}) = g\,\mathbf{V}^{-}_{q}(\mathbf{x}).
\]
\end{proof}

\subsection{Hard Alignment Limit of Kernel Averaging}

\label{hard_alignment_lim}

\begin{tcolorbox}[title=Hard Alignment Limit of Kernel Averaging]
\textbf{Property 2.} Let $p$ be a probability distribution on $\mathbb{R}^d$ with a compact support $\mathrm{supp}(p)$. Define the distance from $\mathbf{x}$ to the support as $d^* = \min_{\mathbf{y} \in \mathrm{supp}(p)} \|\mathbf{x}-\mathbf{y}\|$. 

Assume there is a unique nearest neighbor $\mathbf{y}^*(\mathbf{x}) \in \mathrm{supp}(p)$ such that $\|\mathbf{x} - \mathbf{y}^*(\mathbf{x})\| = d^*$. 

Define the normalized kernel weights using the Laplacian kernel $k(\mathbf{x},\mathbf{y}) = \exp(-\|\mathbf{x}-\mathbf{y}\|/\tau)$:
\[
\tilde{k}_\tau(\mathbf{x},\mathbf{y}) = \frac{k(\mathbf{x},\mathbf{y})}{Z_p}, \quad \text{where} \quad Z_p = \mathbb{E}_{\mathbf{y}' \sim p}\left[k(\mathbf{x},\mathbf{y}')\right]
\]

Then, as $\tau \to 0$, the expectation of $\mathbf{y}$ under these normalized weights converges to the nearest neighbor:
\[
\lim_{\tau \to 0} \mathbb{E}_{\mathbf{y} \sim p}\left[\tilde{k}_\tau(\mathbf{x},\mathbf{y})\,\mathbf{y}\right] = \mathbf{y}^*(\mathbf{x})
\]
\end{tcolorbox}

\begin{proof}
Let $\mathbf{m}_\tau(\mathbf{x}) = \mathbb{E}_{\mathbf{y} \sim p}\left[\tilde{k}_\tau(\mathbf{x},\mathbf{y})\,\mathbf{y}\right]$. We want to show that for any given $\epsilon > 0$, the probability mass induced by $\tilde{k}_\tau$ concentrates entirely within the open ball $B_\epsilon(\mathbf{y}^*) = \{\mathbf{y} \in \mathbb{R}^d : \|\mathbf{y} - \mathbf{y}^*\| < \epsilon\}$ as $\tau \to 0$.

Let $A_\epsilon = \mathrm{supp}(p) \setminus B_\epsilon(\mathbf{y}^*)$ be the set of points outside this $\epsilon$-neighborhood. Because $\mathrm{supp}(p)$ is compact and $\mathbf{y}^*$ is the unique nearest neighbor, the minimum distance from $\mathbf{x}$ to any point in $A_\epsilon$ is strictly greater than $d^*$. Let this minimum distance be $d^* + \gamma$ for some $\gamma > 0$.

Now, let's bound the normalization constant $Z_p = \int_{\mathrm{supp}(p)} \exp(-\|\mathbf{x}-\mathbf{y}'\|/\tau) \, dp(\mathbf{y}')$. 
Take a smaller ball around $\mathbf{y}^*$ of radius $\delta$, where $0 < \delta < \gamma$. By the triangle inequality, for any $\mathbf{y} \in B_\delta(\mathbf{y}^*)$, we have $\|\mathbf{x}-\mathbf{y}\| \leq d^* + \delta$. Since $\mathbf{y}^* \in \mathrm{supp}(p)$, the probability measure of this ball under $p$ is strictly positive: $\int_{B_\delta(\mathbf{y}^*)} dp(\mathbf{y}) > 0$.

We can lower-bound $Z_p$ by only integrating over $B_\delta(\mathbf{y}^*)$:
\[
Z_p \geq \int_{B_\delta(\mathbf{y}^*)} \exp(-\|\mathbf{x}-\mathbf{y}'\|/\tau) \, dp(\mathbf{y}') \geq \exp\left(-\frac{d^* + \delta}{\tau}\right) \int_{B_\delta(\mathbf{y}^*)} dp(\mathbf{y})
\]

Next, we upper-bound the expectation over the ``far'' set $A_\epsilon$. Since every point in $A_\epsilon$ is at least $d^* + \gamma$ away from $\mathbf{x}$:
\[
\int_{A_\epsilon} \exp(-\|\mathbf{x}-\mathbf{y}\|/\tau) \, dp(\mathbf{y}) \leq \exp\left(-\frac{d^* + \gamma}{\tau}\right) \int_{A_\epsilon} dp(\mathbf{y}) \leq \exp\left(-\frac{d^* + \gamma}{\tau}\right)
\]

Now, consider the proportion of the weight assigned to the far set $A_\epsilon$:
\[
\mathbb{P}_{\tilde{k}_\tau}(A_\epsilon) = \frac{1}{Z_p} \int_{A_\epsilon} \exp(-\|\mathbf{x}-\mathbf{y}\|/\tau) \, dp(\mathbf{y})
\]
\[
\mathbb{P}_{\tilde{k}_\tau}(A_\epsilon) \leq \frac{\exp(-(d^* + \gamma)/\tau)}{\exp(-(d^* + \delta)/\tau) \int_{B_\delta(\mathbf{y}^*)} dp(\mathbf{y})} = \frac{\exp(-(\gamma - \delta)/\tau)}{\int_{B_\delta(\mathbf{y}^*)} dp(\mathbf{y})}
\]

Since we chose $\delta < \gamma$, the exponent $-(\gamma - \delta)/\tau$ is strictly negative. As $\tau \to 0$, the numerator goes to $0$ while the denominator $\int_{B_\delta(\mathbf{y}^*)} dp(\mathbf{y})$ is a fixed positive constant. Therefore, $\mathbb{P}_{\tilde{k}_\tau}(A_\epsilon) \to 0$.

This implies that as $\tau \to 0$, the probability measure induced by the normalized weights $\tilde{k}_\tau$ converges weakly to the Dirac measure $\delta_{\mathbf{y}^*}$ centered at $\mathbf{y}^*(\mathbf{x})$. Since the support is compact (ensuring uniform integrability), the expectation converges to the point of concentration:
\[
\lim_{\tau \to 0} \mathbf{m}_\tau(\mathbf{x}) = \mathbf{y}^*(\mathbf{x})
\]
\end{proof}

\subsection{Property of Embedding variant}

\label{sec:injective_embedding}

\setcounter{proposition}{0}

\begin{tcolorbox}[title=Invariance and Regularity]
\begin{proposition}
Let \(\mathcal{X}_N = \{\mathbf{x} \in \mathbb{R}^{N\times 3} : \sum_i x_i = 0\}\), where each \(x_i\) has type \(t_i \in \mathcal{T}\). For \(\alpha \in \mathcal{T}_{\mathrm{pairs}}\), define
\[
\mathcal{D}_\alpha(\mathbf{x}) = \{ \|x_i - x_j\| : \{t_i,t_j\}=\alpha \},
\quad
\phi(\mathbf{x}) = \bigoplus_\alpha \operatorname{sort}(\mathcal{D}_\alpha(\mathbf{x})).
\]

Let \(G = SO(3)\times \hat S_N\). Then:

\textbf{(i) Invariance.} \(\phi\) is \(G\)-invariant:
\[
\phi(\mathbf{x}) = \phi(g \cdot \mathbf{x}), \quad \forall g \in G.
\]

\textbf{(ii) Regularity.} \(\phi\) is piecewise smooth and differentiable almost everywhere, with non-differentiability only when pairwise distances within a type class coincide.
\end{proposition}
\end{tcolorbox}

\begin{proof}

\textbf{(i) $G$-invariance.}

Let $g=(R,\pi)\in SO(3)\times \hat{S}_N$, where $R\in SO(3)$ and $\pi$ is a permutation that preserves atom types. For any pair $(i,j)$, we have
\[
g\cdot x_i = R x_{\pi(i)}, \qquad g\cdot x_j = R x_{\pi(j)}.
\]
Using rotational invariance of the Euclidean norm,
\[
\|R x_{\pi(i)} - R x_{\pi(j)}\| = \|x_{\pi(i)} - x_{\pi(j)}\|.
\]
Since $\pi$ preserves types, the mapping $(i,j)\mapsto (\pi(i),\pi(j))$ induces a bijection within each type class $\alpha \in \mathcal{T}_{\mathrm{pairs}}$, hence
\[
\mathcal{D}_\alpha(g\mathbf{x}) = \mathcal{D}_\alpha(\mathbf{x}).
\]
Finally, sorting is invariant under permutations of elements in a multiset, so
\[
\operatorname{sort}(\mathcal{D}_\alpha(g\mathbf{x})) = \operatorname{sort}(\mathcal{D}_\alpha(\mathbf{x})),
\]
and concatenating over $\alpha$ yields $\phi(g\mathbf{x})=\phi(\mathbf{x})$.

\vspace{0.5em}

\textbf{(ii) Regularity.}

Fix a pair $(i,j)$. The map
\[
\mathbf{x} \mapsto \|x_i - x_j\|
\]
is smooth on the open set $\{x_i \neq x_j\}$ and has smooth gradient
\[
\nabla_{x_i} \|x_i-x_j\| = \frac{x_i-x_j}{\|x_i-x_j\|}, \qquad 
\nabla_{x_j} \|x_i-x_j\| = -\frac{x_i-x_j}{\|x_i-x_j\|}.
\]

Thus, before sorting, all components of $\phi$ are smooth except on the collision sets $\{x_i=x_j\}$. The sorting operator is piecewise smooth: it is smooth on regions where all distances within each $\mathcal{D}_\alpha(\mathbf{x})$ are strictly ordered. Nonsmoothness occurs exactly on the boundaries where two distances coincide:
\[
\|x_i-x_j\| = \|x_k-x_\ell\| \quad \text{within the same type block } \alpha,
\]
or when $x_i=x_j$.

Each such condition defines an algebraic variety of codimension at least one in $\mathbb{R}^{N\times 3}$, hence a measure-zero set. Therefore, $\phi$ is differentiable almost everywhere and piecewise smooth.

\end{proof}

\subsection{Equivariance of Alignment-Based Drift}
\label{sec:equi_drift_alignment}

\begin{tcolorbox}[title=Equivariance of Alignment-Based Drift]
\begin{lemma}
\label{lem:alignment_equivariance}
Let $G = S(N') \times O(3)$ be the symmetry group, acting on a state $\mathbf{x}$ via $g \cdot \mathbf{x} = \Pi R \mathbf{x}$ for $g = (\Pi, R) \in G$. We define the optimally aligned target $g^*(\mathbf{x}, \mathbf{y})$ as the group element that minimizes the distance:
\begin{equation}
    g^*(\mathbf{x}, \mathbf{y}) = \arg\min_{g \in G} \left\| \mathbf{x} - g \cdot \mathbf{y} \right\|
\end{equation}
Assume the drifting field $\mathbf{V}^+_p(\mathbf{x})$ is constructed using an isotropic kernel $k$ over this minimal distance, directing $\mathbf{x}$ toward the optimally aligned targets:
\begin{equation}
    \mathbf{V}^+_p(\mathbf{x}) = \frac{\mathbb{E}_{\mathbf{y}^+ \sim p} \left[ k\left(\mathbf{x}, g^*(\mathbf{x}, \mathbf{y}^+) \cdot \mathbf{y}^+ \right) \left( g^*(\mathbf{x}, \mathbf{y}) \cdot \mathbf{y}^+ - \mathbf{x} \right) \right]}{\mathbb{E}_{\mathbf{y^{+}} \sim p}\left[k(\mathbf{x},g^*(\mathbf{x}, \mathbf{y}^+) \cdot \mathbf{y}^+)\right]} 
\end{equation}
Then the vector field $\mathbf{V}^+_p$ is strictly $G$-equivariant, i.e., $\mathbf{V}^+_p(h \cdot \mathbf{x}) = h \cdot \mathbf{V}^+_p(\mathbf{x})$ for any $h \in G$.
\end{lemma}
\end{tcolorbox}

We prove $G$-equivariance of $\mathbf{V}^+_p$, defined as:
\begin{equation}
    \mathbf{V}^+_p(\mathbf{x}) = \frac{\mathbb{E}_{\mathbf{y}^+ \sim p} \left[ k\left(\mathbf{x}, \mathbf{y}^+ \right) \left( \mathbf{y}^+ - \mathbf{x} \right) \right]}{\mathbb{E}_{\mathbf{y^{+}} \sim p}\left[k(\mathbf{x},\mathbf{y}^+)\right]} 
\end{equation}
the same argument gives $G$-equivariance of $\mathbf{V}^-_q$ with $q$ in place of $p$, and since $\mathbf{V}_{p,q} = \mathbf{V}^+_p - \mathbf{V}^-_q$ with $g$ acting linearly, $G$-equivariance of $\mathbf{V}_{p,q}$ follows.

\begin{proof}
To prove equivariance, we evaluate the drifting field at a transformed input $h \cdot \mathbf{x}$. We must first determine the optimal alignment $g^*(h \cdot \mathbf{x}, \mathbf{y})$ for this transformed state:
\begin{align}
    g^*(h \cdot \mathbf{x}, \mathbf{y}) &= \arg\min_{g \in G} \left\| h \cdot \mathbf{x} - g \cdot \mathbf{y} \right\| \\
    &= \arg\min_{g \in G} \left\| h \cdot \left( \mathbf{x} - h^{-1}g \cdot \mathbf{y} \right) \right\| 
\end{align}
Because the group action $h \in S(N') \times O(3)$ is distance-preserving (orthogonal), we can factor it out of the norm. Letting $g' = h^{-1}g$, the optimization over $g \in G$ is equivalent to an optimization over $g' \in G$:
\begin{equation}
    \arg\min_{g' \in G} \left\| \mathbf{x} - g' \cdot \mathbf{y} \right\| = g^*(\mathbf{x}, \mathbf{y})
\end{equation}
Thus, the optimal group element for the transformed input is left-multiplied by $h$:
\begin{equation}
    g^*(h \cdot \mathbf{x}, \mathbf{y}) = h \cdot g^*(\mathbf{x}, \mathbf{y})
\end{equation}
This gives us two critical properties. First, the minimal distance inside the kernel evaluates to the same value, rendering the kernel weights $G$-invariant:
\begin{align}
    \left\| h \cdot \mathbf{x} - g^*(h \cdot \mathbf{x}, \mathbf{y}) \cdot \mathbf{y} \right\| &= \left\| h \cdot \mathbf{x} - (h \cdot g^*(\mathbf{x}, \mathbf{y})) \cdot \mathbf{y} \right\| \\
    &= \left\| \mathbf{x} - g^*(\mathbf{x}, \mathbf{y}) \cdot \mathbf{y} \right\|
\end{align}
Second, the local difference vector shifts equivariantly:
\begin{align}
    g^*(h \cdot \mathbf{x}, \mathbf{y}) \cdot \mathbf{y} - h \cdot \mathbf{x} &= (h \cdot g^*(\mathbf{x}, \mathbf{y})) \cdot \mathbf{y} - h \cdot \mathbf{x} \\
    &= h \cdot \left( g^*(\mathbf{x}, \mathbf{y}) \cdot \mathbf{y} - \mathbf{x} \right)
\end{align}
Substituting these two properties into the definition of the drifting field $\mathbf{V}(h \cdot \mathbf{x})$:
\begin{align}
    \mathbf{V}^+_p(h \cdot \mathbf{x}) &= \frac{\mathbb{E}_{\mathbf{y}^+ \sim p} \left[ k\left(\mathbf{x}, g^*(\mathbf{x}, \mathbf{y}^+) \cdot \mathbf{y}^+ \right) h \cdot \left( g^*(\mathbf{x}, \mathbf{y}^+) \cdot \mathbf{y}^+ - \mathbf{x} \right) \right]}{\mathbb{E}_{\mathbf{y^{+}} \sim p}\left[k(\mathbf{x},g^*(\mathbf{x}, \mathbf{y}^+) \cdot \mathbf{y}^+)\right]} \\
    &= \frac{h \cdot \mathbb{E}_{\mathbf{y}^+ \sim p} \left[ k\left(\mathbf{x}, g^*(\mathbf{x}, \mathbf{y}^+) \cdot \mathbf{y}^+ \right) \left( g^*(\mathbf{x}, \mathbf{y}^+) \cdot \mathbf{y}^+ - \mathbf{x} \right) \right]}{\mathbb{E}_{\mathbf{y^{+}} \sim p}\left[k(\mathbf{x},g^*(\mathbf{x}, \mathbf{y}^+) \cdot \mathbf{y}^+)\right]}  \\
    &= h \cdot \mathbf{V}^+_p
\end{align}
This concludes the proof that aligning the target distribution to the current state sample-wise yields a strictly $G$-equivariant drifting field.
\end{proof}


\section{Implementation details}
\label{sec:implementation}

\subsection{Data Preprocessing}

\paragraph{Molecular conformer generation.}

We follow ET-Flow~\cite{hassan2024flow} and adopt the same RDKit-based featurization~\cite{riniker2015better}, as summarized in Table~\ref{tab:atomic_features}, for both equivariant and non-equivariant architectures. For edge construction, we combine global, radius-based edges with local edges derived from the molecular graph.

\begin{table}[ht]
\centering
\caption{Atomic features used as initial node embeddings.}
\label{tab:atomic_features}
\begin{tabular}{lll}
\toprule
\textbf{Name} & \textbf{Description} & \textbf{Range} \\
\midrule
chirality & Chirality Tag & \{unspecified, tetrahedral CW \& CCW, other\} \\
degree & Number of bonded neighbors & $\{x : 0 \leq x \leq 10, x \in \mathbb{Z}\}$ \\
charge & Formal charge of atom & $\{x : -5 \leq x \leq 5, x \in \mathbb{Z}\}$ \\
num\_H & Total Number of Hydrogens & $\{x : 0 \leq x \leq 8, x \in \mathbb{Z}\}$ \\
number\_radical\_e & Number of Radical Electrons & $\{x : 0 \leq x \leq 4, x \in \mathbb{Z}\}$ \\
hybridization & Hybridization type & \{sp, sp$^2$, sp$^3$, sp$^3$d, sp$^3$d$^2$, other\} \\
aromatic & Whether on an aromatic ring & \{True, False\} \\
in\_ring & Whether in a ring & \{True, False\} \\
\bottomrule
\end{tabular}
\end{table}

\paragraph{Transition State Prediction.}

We follow GoFlow~\cite{galustian2025goflow} and encode chemical reactions using the condensed graph of reaction (CGR)~\cite{heid2021machine}. The CGR represents a reaction as a single graph obtained by superimposing the molecular graphs of reactants and products. Each atom and bond is assigned dual labels corresponding to its states before and after the reaction, enabling the model to capture changes in bonding, charge, hybridization, and related properties. Initial node and edge features for reactants and products are computed using the RDKit-based featurization described in Table~\ref{tab:atomic_features}. The edge list is constructed by combining global, radius-based edges with local edges derived from the reactant and product molecular graphs.

\subsection{Embedding strategies}
\label{sec:embedding_strategies}

We consider two ways of defining the orbits within which distances are sorted. The first defines them by interaction type, so that the pair $(i,j)$ is labelled by the unordered pair of atomic numbers of its two atoms. All atoms of the same element are then treated as interchangeable and the resulting embedding is invariant under the full group $\hat S_N$ of element-preserving relabellings. This choice requires no graph information and is inexpensive to compute, which makes it the default in our experiments on GEOM-QM9.

Invariance under all of $\hat S_N$ is nevertheless stronger than physically required, since it pools interactions of atoms that share an element but are chemically distinct. The second strategy therefore derives the orbits from the molecular graph, which restricts the invariance to the relabellings that actually leave the molecule unchanged. We use it for the GEOM-DRUGS experiments and in the injectivity ablation of Section~\ref{sec:add_experiments}, and describe it in detail below.

The relabellings that genuinely leave the molecule unchanged are the automorphisms of its attributed molecular graph $G=(V,E,c_V)$, where $V=\{1,\dots,N\}$, $E$ are the covalent bonds, and the vertex colouring $c_V$ is given by the atomic number together with local atom attributes (degree, formal charge, number of attached hydrogens, hybridisation, aromaticity, chirality tag, ring membership),
\begin{equation}
  \mathrm{Aut}(G)=\big\{\pi\in S_N \;\big|\;
    \{i,j\}\in E \Leftrightarrow \{\pi i,\pi j\}\in E,\;
    c_V(\pi i)=c_V(i)\ \forall i\big\}.
\end{equation}
We obtain it from the generators returned by \textsc{nauty} for the coloured graph. The group acts on pairs by $\pi\cdot\{i,j\}=\{\pi i,\pi j\}$, and we label each pair by the canonical representative of its orbit,
\begin{equation}
  \label{eq:orbit-id-aut}
  c_{\mathrm{aut}}(i,j)
  = \min_{\pi\in\mathrm{Aut}(G)}
    \Big[\min(\pi i,\pi j)\cdot N + \max(\pi i,\pi j)\Big],
\end{equation}
that is, the lexicographically smallest pair reachable from $(i,j)$, encoded by the injective map $h(a,b)=aN+b$ for $a<b<N$.

Since the minimum runs over the whole group, $c_{\mathrm{aut}}$ is constant on orbits and distinct across them. Hence $c_{\mathrm{aut}}(i,j)=c_{\mathrm{aut}}(k,l)$ holds exactly when the two pairs are related by a graph symmetry, and $\phi_{c_{\mathrm{aut}}}$ is invariant under $SO(3)\times\mathrm{Aut}(G)$. Because $c_V$ refines the element labelling we have $\mathrm{Aut}(G)\le\hat S_N$, so each interaction type $\alpha$ splits into a disjoint union of automorphism orbits. Only genuinely indistinguishable interactions are therefore sorted together, such as the hydrogens of a methyl group, the two oxygens of a carboxylate, or symmetry-related ring positions.

\subsection{Evaluation Metrics}

\paragraph{Molecular conformer generation.}

In the test set, for each molecule with \(L\) reference conformers, we sample \(K = 2L\) conformers and assess their quality using standard evaluation metrics.

A conformer \(C\) encodes the 3D coordinates of all atoms in a molecular graph, and can be represented as a set of vectors in \(\mathbb{R}^{3n}\). Prior work commonly evaluates conformer generation using Average Minimum RMSD (AMR) and Coverage (COV) in both Precision (P) and Recall (R) settings. For a given molecule, let:
\begin{itemize}
    \item \(\{C_l^*\}_{l=1}^L\) denote the set of ground-truth conformers,
    \item \(\{C_k\}_{k=1}^K\) denote the generated conformers, with \(K = 2L\),
    \item \(\delta\) be a predefined RMSD threshold for matching conformers.
\end{itemize}

\textbf{COV-P}: the fraction of generated conformers that match at least one ground-truth conformer within the threshold:
\[
\text{COV-P} = \frac{1}{K} \left| \left\{ k \in [1, K] \,\middle|\, \exists l \in [1, L], \ \text{RMSD}(C_k, C_l^*) < \delta \right\} \right|
\]

\textbf{AMR-P}: the average minimum RMSD from each generated conformer to its closest ground-truth conformer:
\[
\text{AMR-P} = \frac{1}{K} \sum_{k=1}^{K} \min_{l \in [1,L]} \text{RMSD}(C_k, C_l^*)
\]

\textbf{COV-R}: the fraction of ground-truth conformers that are matched by at least one generated conformer:
\[
\text{COV-R} = \frac{1}{L} \left| \left\{ l \in [1, L] \,\middle|\, \exists k \in [1, K], \ \text{RMSD}(C_k, C_l^*) < \delta \right\} \right|
\]

\textbf{AMR-R}: the average minimum RMSD from each ground-truth conformer to its closest generated conformer:
\[
\text{AMR-R} = \frac{1}{L} \sum_{l=1}^{L} \min_{k \in [1,K]} \text{RMSD}(C_l^*, C_k)
\]

\paragraph{Transition state generation.}

For the TS task, we adopt the sampling scheme of GoFlow~\cite{galustian2025goflow}. Specifically, we generate 25 samples per reaction and select the one closest to the median of the predicted set as the final prediction. Then we compute the RMSD by first aligning the molecules $\mathbf{x}_1$ and $\mathbf{x}_2$ using the Kabsch algorithm and then computing:
\begin{equation}\label{eq:rmsd_definition}
\begin{split}
    \text{RMSD}(\mathbf{x}_1,\mathbf{x}_2) ={}& 
    \sqrt{\frac{\sum_{i=1}^{N} \|\mathbf{x}_{1,i} - \mathbf{x}_{2,i} \|^2}{N}} \\
    ={}& \sqrt{\frac{\sum_{i=1}^{N} \sum_{j \in \{x,y,z\}} (x_{1,i,j} - x_{2,i,j})^2}{N}}
\end{split}
\end{equation}
with $N$ denoting the number of atoms and the index $i$ referring to the $i$-th atom, whose Cartesian coordinates (x,y,z) are included in the summation. To capture more fine-grained geometric deviations, we introduce the bond distance mean absolute error (DMAE), defined as:
\begin{equation}\label{eq:dmae_definition}
    \text{DMAE}(\mathbf{x}_1,\mathbf{x}_2) = \frac{\sum_{i\not= j} |d_{1,ij} - d_{2,ij}|}{N(N-1)}
\end{equation}
with $d_{1,ij}=\|\mathbf{x}_{1,i} - \mathbf{x}_{1,j} \|^2$ as the interatomic distance between atom $i$ and $j$.

\subsection{Chirality Correction}

We follow ET-Flow to perform a simple post-hoc chirality correction on generated molecular conformations. After sampling, we compare the orientation of each tetrahedral center in the generated structure with the required chirality specified by RDKit tags. If a mismatch is detected, we correct it by flipping the corresponding configuration, ensuring consistency of the stereochemistry without modifying the model itself.

We compute the oriented volume for each tetrahedral center as
\begin{equation}
\mathrm{OV}(p_1,p_2,p_3,p_4)
=
\mathrm{sign}
\left(
\begin{vmatrix}
1 & 1 & 1 & 1 \\
x_1 & x_2 & x_3 & x_4 \\
y_1 & y_2 & y_3 & y_4 \\
z_1 & z_2 & z_3 & z_4
\end{vmatrix}
\right).
\end{equation}

Let $\mathrm{OV}_{\text{gen}}$ denote the oriented volume of the generated conformation and $\mathrm{OV}_{\text{ref}}$ the target chirality from RDKit. If
\[
\mathrm{OV}_{\text{gen}} \neq \mathrm{OV}_{\text{ref}},
\]
we apply a reflection (e.g., flipping the configuration along the $z$-axis) to obtain the corrected conformation $\hat{x}$:
\[
\hat{x} = \mathcal{F}(x),
\]
where $\mathcal{F}$ denotes the chirality correction operation.

\subsection{Training Details and Hyperparameters}
\label{sec:train_details}

\paragraph{Drifting hyperparameters.}

When computing the drifting field, we follow~\citet{deng2026generative} and use three temperature values, $\tau \in \{0.02, 0.05, 0.2\}$. Smaller temperatures yield sharper, more localized interactions, while larger temperatures help prevent mode collapse. When combining the drifting fields across temperatures, each component $V_{p,q_{\theta}}^\tau$ is normalized to avoid dominance of low-temperature contributions. During training, we sample $N^{-} = 64$ negative samples per class, unless stated otherwise. The number of positive samples $N^{+}$ depends on the task: for conformer generation, up to 30 samples are used, whereas for transition state prediction, $N^{+} = 1$. The number of sampled classes is adjusted based on the model architecture to satisfy memory constraints. Additionally, we employ a two-sided normalization of the kernel. Specifically, the softmax is applied independently over $\mathbf{y}$ and $\mathbf{x}$, and the resulting normalized kernels are combined via their geometric mean:
\begin{equation}
\tilde{k}(\mathbf{x},\mathbf{y}) = \sqrt{\tilde{k}(\mathbf{x},\mathbf{y})_{\mathbf{x}} , \tilde{k}(\mathbf{x},\mathbf{y})_{\mathbf{y}} }.
\end{equation}
This kernel is subsequently re-normalized over all $\mathbf{y}$ for each fixed $\mathbf{x}$, following~\citet{deng2026generative}.

\paragraph{Molecular conformation generation.}

We follow ET-Flow~\cite{hassan2024flow} and employ a harmonic prior distribution defined by the molecular graph connectivity. To enforce translational invariance, the center of mass is removed from each sampled configuration.

For the equivariant ET-Flow architecture, we adopt the hyperparameters originally proposed by the authors. For GEOM-QM9, we train for a fixed number of 200 epochs with a batch size of $N_\mathrm{c}=8$, using the full dataset per epoch on 2 A100 GPUs. We use the AdamW optimizer with no weight decay and employ a cosine annealing learning rate schedule, decaying from a maximum of $8 \cdot 10^{-4}$ to a minimum of $10^{-7}$ over 100 epochs; thereafter, the maximum learning rate is reduced by a factor of 0.05 for the remaining training. 

When using the non-equivariant DiTMC architecture~\cite{frank2025sampling}, we reimplemented the original JAX-based model in PyTorch, resulting in minor modifications. In particular, we removed the graph geodesic embedding and employed the multi-head attention module from PyTorch Geometric. To obtain a fully non-equivariant architecture, we incorporate \textbf{a}bsolute positional, atom \textbf{i}ndex, and \textbf{r}elative distance embeddings (airPE), thereby explicitly breaking permutation, rotational, and translational equivariance. Apart from these changes, we follow the hyperparameter settings reported in the original work. We train \method using the AdamW optimizer with a weight decay of $0.01$ for 250 epochs. For GEOM-QM9, we employ a learning rate schedule with a maximum of $3 \cdot 10^{-4}$ with a batch size of $N_\mathrm{c}=36$ on 1 H200 GPU. We linearly warm up the learning rate from an initial value of $10^{-5}$ to the respective maximum over the first 1\% of training steps. Subsequently, we apply a cosine decay schedule, reducing the learning rate to a minimum of $0$ for GEOM-QM9. For GEOM-DRUGS, the maximum learning rate is set to $1 \cdot 10^{-4}$ and decays to a minimum of $10^{-6}$ over 50 epochs, after which the maximum learning rate is reduced by a factor of $0.1$ for the remaining training. We use a batch size of $N_\mathrm{c}=4$ molecular graphs with $N^{-}=64$ negative samples each per GPU, and train for 100 epochs on 8 A100 GPUs.

\paragraph{Transition state generation.}

We follow TSDiff~\cite{kim2024diffusion} and GoFlow~\cite{galustian2025goflow} and employ a Gaussian prior distribution, with the center of mass removed from each sample to ensure translational invariance. We adopt the original model architecture introduced in GoFlow, while increasing the model capacity and tuning the hyperparameters, thereby scaling the base (B) model to a larger (L) variant. We train with a batch size of $N_\mathrm{c}=4$ on a single A100 GPU using the AdamW optimizer with a weight decay of $1.24 \cdot 10^{-6}$ and a constant learning rate of $2.7 \cdot 10^{-5}$. The model is trained for 650 epochs, and we use the final checkpoint for inference.

Table~\ref{tab:hyperparameters_etflow} summarizes the relevant architectural hyperparameters used.

\begin{table}[ht]
\centering
\caption{Hyperparameters for ET-Flow, DiTMC-airPE, and GoFlow}
\label{tab:hyperparameters_etflow}
\begin{tabular}{llllll}
\toprule
Hyperparameter & ET-Flow & DiTMC+airPE (B) & DiTMC+airPE (L) & GoFlow (B) & GoFlow (L)\\
\midrule
num\_layers & 20 & 6 & 8 & 3 & 12 \\
hidden\_channels & 160 & 256 & 384 & 256 & 256\\
num\_heads & 8 & 8 & 12 & 8 & 8 \\
cutoff & 11.0 & 11.0 & 11.0 & 11.0 & 11.0 \\
node\_attr\_dim & 10 & 10 & 10 & 10 & 10\\
edge\_attr\_dim & 1 & 1 & 1 & 1 & 1\\
num\_layers\_mgn & - & 2 & 3 & - & -\\
\midrule
\# param & 8.3M & 8.9M & 22.2M & 5.0M & 18.7M \\
\bottomrule
\end{tabular}
\end{table}

\paragraph{Training cost.} 

Because drifting optimizes a distribution-matching objective rather than a per-sample regression target, its training is more expensive than that of flow matching (FM). Table~\ref{tab:training_cost} contrasts this overhead with the corresponding inference speed-ups on GEOM-QM9. Training ET-Flow with FM for 200 epochs at the recommended per-GPU batch size of 128 on 2 A100 GPUs takes roughly 12 hours. For drifting, the per-GPU batch of $N^{-} \times N_c$ is capped by memory, so the number of negatives $N^{-}$ dominates the training time, as increasing it reduces the number of distinct molecular graphs processed per step. With $N^{-}=8$, training takes $2.4\times$ longer while sampling becomes $40\times$ faster. The symmetrization itself is comparatively cheap and increases training time by only $7\%$ over BaseDrift.

\begin{table}[h]
\centering
\caption{Training cost versus inference speed-up for drifting relative to FM. Training times are wall-clock on 2 A100 GPUs for GEOM-QM9. Relative training cost and inference speed-up are measured against the flow-matching baseline of the same architecture.}
\label{tab:training_cost}
\begin{tabular}{lccc}
\toprule
Method & Train time & Rel. train cost & Inference speed-up \\
\midrule
FM w. ET-Flow & 12 h & $1\times$ & $1\times$ \\
\midrule
BaseDrift w. ET-Flow ($N^{-}=8$) & 27 h & $2.3\times$ & $40\times$ \\
\method w. ET-Flow ($N^{-}=8$) & 29 h & $2.4\times$ & $40\times$ \\
\bottomrule
\end{tabular}
\end{table}


\section{Training and Sampling algorithm}

Algorithm~\ref{alg:method_training} describes the training of \method in both settings, namely drifting in Cartesian space and in the $G$-invariant embedding space using an embedding function $\phi$. For both methods, the generated and data samples are mean-centered to guarantee translational invariance.

\begin{algorithm}[H]
\caption{Training of \method.}
\label{alg:method_training}
\begin{algorithmic}[1]
\Require (Untrained) generator $f_\theta$; embedding function $\phi$, data distribution $p(\,\cdot\mid c)$; prior $p_\epsilon$; kernel function $k$; mode; temperatures $\mathcal{T}$; $N_c, N^+, N^-$.
\While{not converged}
  \State Sample $N_c$ classes; for each $c$:
  \State \quad Sample $\{\mathbf{y}^+_j\}_{j=1}^{N^+} \sim p(\,\cdot \mid c)$,
                $\{\boldsymbol\epsilon_i\}_{i=1}^{N^-} \sim p_\epsilon$.
  \State \quad $\mathbf{x}_i \leftarrow f_\theta(\boldsymbol\epsilon_i \mid c)$,
                \quad $\mathbf{y}^-_k \leftarrow \mathrm{stopgrad}(\mathbf{x}_k)$.
  \State $\mathbf{x}_i \leftarrow \mathbf{x}_i - \tfrac{1}{N}\sum_{a=1}^{N}\mathbf{x}_{i,a}$
  \If{mode $=\textsc{cartesian}$}
    \State $\mathbf{u}_i \leftarrow \mathbf{x}_i$,\;
           $\mathbf{v}^\pm_{i\ell} \leftarrow g^*(\mathbf{x}_i, \mathbf{y}^\pm_\ell) \cdot \mathbf{y}^\pm_\ell$.
  \Else
    \State $\mathbf{u}_i \leftarrow \phi(\mathbf{x}_i)$,\;
           $\mathbf{v}^\pm_{i\ell} \leftarrow \phi(\mathbf{y}^\pm_\ell)$.
  \EndIf
  \State $\displaystyle \mathbf{V}_i \leftarrow
            \sum_{\tau\in\mathcal{T}}
            \Big[\sum_j \tilde k_\tau(\mathbf{u}_i, \mathbf{v}^+_{ij})(\mathbf{v}^+_{ij}-\mathbf{u}_i)
                - \sum_k \tilde k_\tau(\mathbf{u}_i, \mathbf{v}^-_{ik})(\mathbf{v}^-_{ik}-\mathbf{u}_i)\Big]$.
  \State $\mathcal{L} \leftarrow \frac{1}{N_c N^-}\sum_{c,i}
              \big\|\mathbf{u}_i - \mathrm{stopgrad}(\mathbf{u}_i + \mathbf{V}_i)\big\|^2$.
  \State $\theta \leftarrow \theta - \eta\,\nabla_\theta \mathcal{L}$.
\EndWhile
\end{algorithmic}
\end{algorithm}

Algorithm~\ref{alg:method_sampling} describes the one-shot inference procedure of \method. For the conformer generation task, we sample $2K$ conformers, where $K$ denotes the number of ground-truth conformers for a given class $c$. For the TS task, we adopt the sampling scheme of GoFlow~\cite{galustian2025goflow}. Specifically, we generate 25 samples and select the one closest to the median of the predicted set as the final prediction.

\begin{algorithm}[H]
\caption{One-shot sampling with \method.}
\label{alg:method_sampling}
\begin{algorithmic}[1]
\Require Generator $f_\theta$; prior $p_\epsilon$; class $c$; number of samples $K$.
\State Sample $\boldsymbol\epsilon_1,\dots,\boldsymbol\epsilon_K \sim p_\epsilon$;\;
       $\mathbf{x}_i \leftarrow f_\theta(\boldsymbol\epsilon_i \mid c)$.
\State \Return $\{\mathbf{x}_i\}_{i=1}^K$ \,(conformer task)\, or
       $\mathbf{x}_{i^*}$ with $i^* = \arg\min_i \|\mathbf{x}_i - \mathrm{median}_j\,\mathbf{x}_j\|$ \,(TS task).
\end{algorithmic}
\end{algorithm}


\section{Additional experiments}
\label{sec:add_experiments}



\paragraph{Generation diversity}

To measure diversity directly we compute the heavy-atom RMSD between every pair of generated conformers of the same molecule, and report the mean, minimum and maximum of this per-molecule distribution alongside the same statistics for the ground-truth ensemble and for the multi-step ET-Flow baseline (Tables~\ref{tab:conformer-diversity} and \ref{tab:conformer-diversity-drugs}). \method{} is comparable to ET-Flow across all three statistics, and both stay close to the reference. Mode collapse would contract the whole pairwise-RMSD distribution, which we do not observe. The larger minimum on GEOM-DRUGS points the other way, indicating fewer near-duplicate samples rather than tighter clustering.

\begin{table}[ht]
\centering
\caption{Conformer diversity on GEOM-QM9, measured as the pairwise heavy-atom RMSD between every pair of conformers belonging to the same molecules. Per-molecule statistics are aggregated over all SMILES by mean and median. Values closest to the ground-truth ensemble are preferred.}
\label{tab:conformer-diversity}
\begin{adjustbox}{max width=\linewidth}
\begin{tabular}{lcccccc}
\toprule
& \multicolumn{2}{c}{Avg. RMSD (\AA)} & \multicolumn{2}{c}{Min. RMSD (\AA)} & \multicolumn{2}{c}{Max. RMSD (\AA)} \\
\cmidrule(lr){2-3} \cmidrule(lr){4-5} \cmidrule(lr){6-7}
Method & Mean & Median & Mean & Median & Mean & Median \\
\midrule
GEOM-QM9 (reference) & 0.606 & 0.588 & 0.187 & 0.041 & 0.933 & 0.826 \\
\midrule
ET-Flow~\cite{hassan2024flow} (50 NFE) & {0.532} & {0.540} & {0.106} & {0.020} & {0.877} & {0.834} \\
\method$_\text{DiTMC+airPE (B)}$ & {0.533} & {0.525} & {0.118} & {0.041} & {0.830} & {0.777} \\
\bottomrule
\end{tabular}
\end{adjustbox}
\end{table}

\begin{table}[ht]
\centering
\caption{Conformer diversity on GEOM-DRUGS, measured as the pairwise heavy-atom RMSD between every pair of conformers belonging to the same molecules. Per-molecule statistics are aggregated over all molecules by mean and median. Values closest to the ground-truth ensemble are preferred.}
\label{tab:conformer-diversity-drugs}
\begin{adjustbox}{max width=\linewidth}
\begin{tabular}{lcccccc}
\toprule
& \multicolumn{2}{c}{Avg. RMSD (\AA)} & \multicolumn{2}{c}{Min. RMSD (\AA)} & \multicolumn{2}{c}{Max. RMSD (\AA)} \\
\cmidrule(lr){2-3} \cmidrule(lr){4-5} \cmidrule(lr){6-7}
Method & Mean & Median & Mean & Median & Mean & Median \\
\midrule
GEOM-DRUGS (reference) & 1.714 & 1.760 & 0.128 & 0.043 & 2.957 & 3.031 \\
\midrule
ET-Flow~\cite{hassan2024flow} (50 NFE) & {1.658} & {1.715} & {0.134} & {0.075} & {2.953} & {3.067} \\
\method$_\text{DiTMC+airPE (L)}$ & {1.657} & {1.732} & {0.335} & {0.290} & {2.937} & {3.052} \\
\bottomrule
\end{tabular}
\end{adjustbox}
\end{table}

\paragraph{Drifting hyperparameters.}

To better understand the hyperparameters introduced by drift-based training, we conduct a series of ablation studies. First, we analyze the effect of the number of negative samples $N^{-}$ generated per molecular graph $N_\text{c}$ during training. Table~\ref{tab:ablation_drifting_n_neg_geom_qm9} compares different settings using an equivariant ET-flow architecture. For each molecular graph, up to 30 conformers are used as positive samples, noting that some graphs contain fewer conformers. The effective batch size, given by $N^{-} \times N_\text{c}$, is kept constant across all experiments to ensure a fair comparison. We observe that generation quality consistently improves as the number of negative samples increases. This effect is particularly pronounced in precision metrics, indicating that increasing $N^{-}$ leads to higher-quality generated conformers during training.

\begin{table}[ht]
\centering
\caption{Ablation of number of negative samples for molecule conformer generation on GEOM-QM9 ($\delta = 0.5,\text{\AA}$). “-R” denotes Recall, and “-P” denotes Precision. All drift models are trained with single-sided normalization.}
\label{tab:ablation_drifting_n_neg_geom_qm9}
\begin{adjustbox}{max width=\linewidth}
\begin{tabular}{lcccccccccc}
\toprule
& & Inference & \multicolumn{2}{c}{Coverage-R (\%) $\uparrow$} & \multicolumn{2}{c}{AMR-R (\AA) $\downarrow$} 
& \multicolumn{2}{c}{Coverage-P (\%) $\uparrow$} & \multicolumn{2}{c}{AMR-P  (\AA) $\downarrow$} \\
\cmidrule(lr){4-5} \cmidrule(lr){6-7} \cmidrule(lr){8-9} \cmidrule(lr){10-11}
Method & $N_\text{c}$ & $N_\text{neg}$ & Mean & Median & Mean & Median & Mean & Median & Mean & Median \\
\midrule
\method$_\text{ET-Flow}$ & 128 & 8 & 94.9 & 100.0 & 0.140 & 0.089 & 87.7 & 100.0 & 0.205 & 0.148 \\
\method$_\text{ET-Flow}$ & 64 & 16 & 94.8 & 100.0 & 0.132 & 0.083 & 88.7 & 100.0 & 0.189 & 0.128 \\
\method$_\text{ET-Flow}$ & 32 & 32 & 95.6 & 100.0 & 0.111 & 0.066 & 91.7 & 100.0 & 0.147 & 0.097 \\
\method$_\text{ET-Flow}$ & 16 & 64 & 95.7 & 100.0 & 0.123 & 0.062 & 91.4 & 100.0 & 0.147 & 0.093 \\
\bottomrule
\end{tabular}
\end{adjustbox}
\end{table}

Additionally, we investigate two kernel normalization strategies for $k$. The first is a “one-sided” normalization, where kernel weights are normalized independently over the $\mathbf{y}$ samples for each generated sample $\mathbf{x}$. In the “two-sided” scheme, kernel values are first normalized over all $\mathbf{x}$ samples for each $\mathbf{y}_i$, and the final kernel is obtained as the geometric mean of the normalized contributions from the $\mathbf{x}$- and $\mathbf{y}$-side. Table~\ref{tab:ablation_drifting_norm_geom_qm9} compares both strategies for the non-equivariant DiTMC as well as the equivariant ET-Flow architecture. We observe that the two-sided normalization yields a consistent, albeit modest, improvement in overall generation quality across both architectures.

\begin{table}[ht]
\centering
\caption{Ablation of kernel normalization strategies for molecule conformer generation on GEOM-QM9 ($\delta = 0.5,\text{\AA}$). “-R” denotes Recall, and “-P” denotes Precision.}
\label{tab:ablation_drifting_norm_geom_qm9}
\begin{adjustbox}{max width=\linewidth}
\begin{tabular}{lccccccccc}
\toprule
&  & \multicolumn{2}{c}{Coverage-R (\%) $\uparrow$} & \multicolumn{2}{c}{AMR-R (\AA) $\downarrow$} 
& \multicolumn{2}{c}{Coverage-P (\%) $\uparrow$} & \multicolumn{2}{c}{AMR-P  (\AA) $\downarrow$} \\
\cmidrule(lr){3-4} \cmidrule(lr){5-6} \cmidrule(lr){7-8} \cmidrule(lr){8-10}
Method & Drift normalization & Mean & Median & Mean & Median & Mean & Median & Mean & Median \\
\midrule
\method$_\text{ET-Flow}$ & one-sided & 95.7 & 100.0 & 0.123 & 0.062 & 91.4 & 100.0 & 0.147 & 0.093 \\
\method$_\text{ET-Flow}$ & two-sided & 95.7 & 100.0 & 0.118 & 0.059 & 92.0 & 100.0 & 0.141 & 0.085 \\
\method$_\text{DiTMC+airPE}$ & one-sided & 95.5 & 100.0 & 0.109 & 0.061 & 92.9 & 100.0 & 0.133 & 0.078 \\
\method$_\text{DiTMC+airPE}$ & two-sided & 95.6 & 100.0 & 0.106 & 0.058 & 93.7 & 100.0 & 0.124 & 0.067 \\
\bottomrule
\end{tabular}
\end{adjustbox}
\end{table}

\paragraph{Expressivity requirements of one-shot models.}

As drift-based generation is inherently a one-shot process, it requires a substantially more expressive model than FM, which distributes the transformation across multiple iterative evaluations and therefore reduces the per-step complexity. To investigate this effect, Table~\ref{tab:ablation_same_model_rdb7} compares drifting with the original GoFlow architecture (B), matching the single-step FM baseline, and with an enlarged variant (L). Using the smaller architecture leads to a severe performance degradation in the one-step FM setting, rendering the predictions largely unusable, whereas drifting still produces reasonable results. Since naively truncating a FM model to a single step is a weak baseline, we additionally fine-tune the smaller architecture (B) via reflow, the standard approach for distilling a multi-step FM model into a single-step generator, using a generated dataset of four noise-sample pairs. This yields a substantial improvement over the one-step teacher, but still falls short of the drifting model trained directly for one-shot generation. Together, these results suggest that one-shot generation requires learning a significantly more complex mapping, effectively demanding a higher Lipschitz constant than iterative FM. Nevertheless, even with the larger architecture, a single forward pass remains faster than repeatedly evaluating a smaller model iteratively.

\begin{table*}[h]
\centering
\caption{Ablation on model size for the TS task on the RDB7 dataset.}
\vspace{0.5em}
\label{tab:ablation_same_model_rdb7}
\begin{adjustbox}{max width=\linewidth}
\begin{tabular}{@{\extracolsep\fill}lllllllll}
\toprule
Method & NFE & Inference &\multicolumn{2}{c}{RMSD (\AA)} & \multicolumn{2}{c}{DMAE (\AA)} \\
\cmidrule(lr){4-5} \cmidrule(lr){6-7}
& & time (ms) & Mean & Median & Mean & Median \\
\midrule
GoFlow (B) & 1 & 0.7 & 0.758 & 0.756 & 0.526 & 0.530 \\
GoFlow (B) $_\text{with reflow}$  & 1 & 0.7 & 0.500 & 0.482 & 0.233 & 0.213 \\
\method$_\text{GoFlow (B)}$ & 1 & 0.7 & 0.480 & 0.437 & 0.209 & 0.170 \\
\method$_\text{GoFlow (L)}$ & 1 & 1.3 & 0.328 & 0.239 & 0.134 & 0.092 \\
\bottomrule
\end{tabular}
\end{adjustbox}
\end{table*}

\paragraph{Different realizations of \method.}
\label{sec:different_realizations}

Table~\ref{tab:ablation_full_geom_qm9} extends Figure~\ref{fig:ablation_drift_realizations} by reporting the numerical results for each \method realization on the GEOM-QM9 dataset. In the iterative setting, an initial rotation was applied prior to the iterative permutation and rotational alignment. 

Because the non-equivariant baseline lacks architectural symmetry constraints, it remains flexible enough to fit the unsymmetrized empirical data distribution $p$. In contrast, the rigidly constrained equivariant architecture struggles to learn this distribution without relying on data augmentation. By introducing a symmetrized drift, we eliminate the structural ambiguities present in the raw data. This progressively enforces symmetry and provides a significantly cleaner learning signal for both architectures.

\begin{table}[ht]
\centering
\caption{Ablation of different realizations of \method for molecule conformer generation results on GEOM-QM9 ($\delta$= 0.5\AA). “-R” denotes Recall and “-P” denotes Precision. All drift models are trained with single-sided normalization.}
\label{tab:ablation_full_geom_qm9}
\begin{adjustbox}{max width=\linewidth}
\begin{tabular}{lcccccccc}
\toprule
& \multicolumn{2}{c}{Coverage-R (\%) $\uparrow$} & \multicolumn{2}{c}{AMR-R (\AA) $\downarrow$} 
& \multicolumn{2}{c}{Coverage-P (\%) $\uparrow$} & \multicolumn{2}{c}{AMR-P  (\AA) $\downarrow$} \\
\cmidrule(lr){2-3} \cmidrule(lr){4-5} \cmidrule(lr){6-7} \cmidrule(lr){8-9}
Method & Mean & Median & Mean & Median & Mean & Median & Mean & Median \\
\midrule
BaseDrift$_\text{ET-Flow, $n_\text{aug}=0$}$ & 0.0 & 0.0 & 1.210 & 1.101 & 0.0 & 0.0 & 1.317 & 1.238 \\
\method$_\text{ET-Flow, rot. align}$ & 49.2 & 50.0 & 0.463 & 0.495 & 41.8 & 37.5 & 0.518 & 0.551 \\
\method$_\text{ET-Flow, iter. align}$ & 92.3 & 100.0 & 0.232 & 0.227 & 83.9 & 96.0 & 0.290 & 0.299 \\
\method$_\text{ET-Flow, embedded no perm.}$ & 94.5 & 100.0 & 0.181 & 0.150 & 87.1 & 98.1 & 0.250 & 0.232 \\
\method$_\text{ET-Flow, embedded}$ & 95.7 & 100.0 & 0.123 & 0.062 & 91.4 & 100.0 & 0.147 & 0.093 \\
\midrule
BaseDrift$_\text{DiTMC+airPE, $n_\text{aug}=0$}$ & 73.0 & 90.2 & 0.352 & 0.338 & 82.9 & 100.0 & 0.291 & 0.242 \\
\method$_\text{DiTMC+airPE, rot. align}$ & 89.9 & 100.0 & 0.180 & 0.137 & 94.4 & 100.0 & 0.142 & 0.085 \\
\method$_\text{DiTMC+airPE, iter. align}$ & 91.2 & 100.0 & 0.222 & 0.178 & 85.5 & 100.0 & 0.257 & 0.201 \\
\method$_\text{DiTMC+airPE, embedded no perm.}$ & 96.0 & 100.0 & 0.117 & 0.071 & 94.1 & 100.0 & 0.137 & 0.085 \\
\method$_\text{DiTMC+airPE, embedded}$ & 95.5 & 100.0 & 0.109 & 0.061 & 92.9 & 100.0 & 0.133 & 0.078 \\
\bottomrule
\end{tabular}
\end{adjustbox}
\end{table}

To compare the exact global solution of the alignment problem in Equation~\ref{eq:sym_aware_diff} with iterative and approximate alternatives, we construct a subset of GEOM-QM9 in which no atom type occurs more than three times. This restriction makes brute-force enumeration feasible, yielding 416 training and 50 test molecular graphs. Table~\ref{tab:ablation_feasible_geom_qm9} shows that, on this subset, the brute-force solution achieves the best performance, closely followed by the iterative approximation. Notably, the performance of the iterative method depends significantly on the initialization strategy: starting with a rotational alignment before applying the Hungarian matching leads to substantially different results. We further observe that, due to the limited permutational freedom in this subset, the benefits of enforcing permutation equivariance are less pronounced than on the full GEOM-QM9 dataset. This is reflected in the small performance gap between the iterative approach and the purely rotationally aligned variant of the drift.

Additionally, we investigate variants of the base drift model in combination with data augmentation. For each positive sample $\mathbf{y}^{+}$, we generate $n_\text{aug}$ augmented targets by applying rotations and permutations consistent with graph automorphisms. We find that data augmentation is essential for the equivariant architecture to learn the drift objective at all: without augmentation, the model fails entirely, whereas with $n_\text{aug}=10$ it begins to capture the signal. Increasing the number of augmentations further improves performance, although the results still lag behind those obtained with \method. This improvement, however, comes at a significant computational cost. Using $n_\text{aug}=10000$ increases the training time from 45 minutes (no augmentation) to approximately 22 hours on a single A100 GPU. For comparison, training with iterative alignment requires about 2 hours, while brute-force alignment takes roughly 16 hours.

\begin{table}[h]
\centering
\caption{Ablation of alignment strategies for molecular conformer generation on a reduced GEOM-QM9 subset ($\delta = 0.5,\text{\AA}$), for which brute-force permutation matching is computationally feasible. “-R” denotes Recall and “-P” denotes Precision. All drift models are trained with single-sided normalization.}
\label{tab:ablation_feasible_geom_qm9}
\begin{adjustbox}{max width=\linewidth}
\begin{tabular}{lcccccccc}
\toprule
& \multicolumn{2}{c}{Coverage-R (\%) $\uparrow$} & \multicolumn{2}{c}{AMR-R (\AA) $\downarrow$} 
& \multicolumn{2}{c}{Coverage-P (\%) $\uparrow$} & \multicolumn{2}{c}{AMR-P  (\AA) $\downarrow$} \\
\cmidrule(lr){2-3} \cmidrule(lr){4-5} \cmidrule(lr){6-7} \cmidrule(lr){8-9}
Method & Mean & Median & Mean & Median & Mean & Median & Mean & Median \\
\midrule
BaseDrift$_\text{ET-Flow, $n_\text{aug}=0$}$ & 2.0 & 0.0 & 1.292 & 1.308 & 2.0 & 0.0 & 1.369 & 1.374 \\
BaseDrift$_\text{ET-Flow, $n_\text{aug}=10$}$ & 67.2 & 100.0 & 0.443 & 0.454 & 50.0 & 50.0 & 0.509 & 0.499 \\
BaseDrift$_\text{ET-Flow, $n_\text{aug}=100$}$ & 79.0 & 100.0 & 0.325 & 0.259 & 71.6 & 100.0 & 0.383 & 0.351 \\
BaseDrift$_\text{ET-Flow, $n_\text{aug}=1000$}$ & 68.3 & 100.0 & 0.397 & 0.348 & 59.9 & 87.5 & 0.470 & 0.422 \\
BaseDrift$_\text{ET-Flow, $n_\text{aug}=10000$}$ & 69.7 & 100.0 & 0.411 & 0.327 & 60.3 & 72.5 & 0.484 & 0.443\\
\method$_\text{ET-Flow, rot. align}$ & 91.7 & 100.0 & 0.141 & 0.047 & 92.0 & 100.0 & 0.140 & 0.055 \\
\method$_\text{ET-Flow, iter. align, perm. first}$ & 88.2 & 100.0 & 0.290 & 0.303 & 85.8 & 100.0 & 0.322 & 0.304 \\
\method$_\text{ET-Flow, iter. align, rot. first}$ & 92.6 & 100.0 & 0.130 & 0.044 & 94.7 & 100.0 & 0.126 & 0.047 \\
\method$_\text{ET-Flow, brute force align}$ & 95.8 & 100.0 & 0.093 & 0.036 & 96.6 & 100.0 & 0.092 & 0.040 \\
\bottomrule
\end{tabular}
\end{adjustbox}
\end{table}

\paragraph{Orbit definition.}

Aliasing in our embedding arises from the sorting step, as distances within an orbit are sorted to make the representation invariant, which discards their assignment to individual atoms and can therefore map two structurally different conformers onto the same embedding. The granularity of the orbits controls how much structural information is lost, and we compare two definitions: atom-type orbits, which treat all atoms of the same element as interchangeable, and graph-automorphism orbits, which are derived from the molecular graph and are strictly finer.

To quantify how often distinct conformations actually collide, we analyze the ground-truth ensembles of GEOM-QM9. For the $30$ most Boltzmann-relevant conformers per molecule we count aliased pairs, i.e.\ pairs whose embeddings agree within $\epsilon$ while their RMSD exceeds $\tau$. Under atom-type orbits, aliasing is rare at the level of individual pairs ($0.4\%$), but since each molecule contributes many conformers, $11.3\%$ of molecules contain at least one aliased pair. The finer graph-automorphism orbits reduce both rates, to $0.2\%$ of pairs and $6.3\%$ of molecules, showing that most residual aliasing stems from the coarseness of the atom-type partition. Training an additional model with graph-automorphism orbits nevertheless yields only a minor empirical difference (Table~\ref{tab:ablation_injectivity}).

\begin{table}[ht]
\centering
\caption{Ablation of the injectivity of the embedding function for molecular conformer generation on the GEOM-QM9 subset ($\delta = 0.5$~\AA). ``-R'' denotes Recall, ``-P'' denotes Precision.}
\label{tab:ablation_injectivity}
\begin{adjustbox}{max width=\linewidth}
\begin{tabular}{lcccccccc}
\toprule
& \multicolumn{2}{c}{COV-R (\%) $\uparrow$} & \multicolumn{2}{c}{AMR-R (\AA) $\downarrow$} & \multicolumn{2}{c}{COV-P (\%) $\uparrow$} & \multicolumn{2}{c}{AMR-P (\AA) $\downarrow$} \\
\cmidrule(lr){2-3} \cmidrule(lr){4-5} \cmidrule(lr){6-7} \cmidrule(lr){8-9}
Orbit definition & Mean & Median & Mean & Median & Mean & Median & Mean & Median \\
\midrule
Graph automorphism & 95.8 & 100.0 & 0.111 & 0.064 & 94.0 & 100.0 & 0.128 & 0.073 \\
Atomic type & 95.6 & 100.0 & 0.106 & 0.058 & 93.7 & 100.0 & 0.124 & 0.067 \\
\bottomrule
\end{tabular}
\end{adjustbox}
\end{table}

To test whether aliasing induces averaged or otherwise degenerate structures, we additionally analyze the distribution of internal coordinates of the GEOM-QM9 test-set predictions, namely bond distances, bond angles and torsion angles. Agreement with the reference ensembles is quantified by the Wasserstein distance, which compares the full distributions rather than only their means. For the $86$ of $1{,}000$ test molecules whose ground-truth ensembles contain at least one aliased conformer pair, we obtain $0.03$~\AA{} for bonded distances, $2.7^\circ$ for bond angles and $7.6^\circ$ for torsions, against $0.07$~\AA{}, $3.0^\circ$ and $5.7^\circ$ for the remaining $914$ molecules. The two groups are of the same order throughout, and the affected molecules are in fact closer to the reference on bonds and angles and only further on torsions.



\clearpage
\section*{NeurIPS Paper Checklist}

\begin{enumerate}

\item {\bf Claims}
    \item[] Question: Do the main claims made in the abstract and introduction accurately reflect the paper's contributions and scope?
    \item[] Answer: \answerYes{} 
    \item[] Justification: The paper claims that equivariant models trained on empirical, non-invariant datasets do not necessarily recover the corresponding symmetrized data distribution. To address this limitation, we introduce two methods and show state-of-the-art results on conformer and transition state generation benchmarks.
    \item[] Guidelines:
    \begin{itemize}
        \item The answer \answerNA{} means that the abstract and introduction do not include the claims made in the paper.
        \item The abstract and/or introduction should clearly state the claims made, including the contributions made in the paper and important assumptions and limitations. A \answerNo{} or \answerNA{} answer to this question will not be perceived well by the reviewers. 
        \item The claims made should match theoretical and experimental results, and reflect how much the results can be expected to generalize to other settings. 
        \item It is fine to include aspirational goals as motivation as long as it is clear that these goals are not attained by the paper. 
    \end{itemize}

\item {\bf Limitations}
    \item[] Question: Does the paper discuss the limitations of the work performed by the authors?
    \item[] Answer: \answerYes{} 
    \item[] Justification: Limitations are discussed in Section~\ref{sec:discussion}.
    \item[] Guidelines:
    \begin{itemize}
        \item The answer \answerNA{} means that the paper has no limitation while the answer \answerNo{} means that the paper has limitations, but those are not discussed in the paper. 
        \item The authors are encouraged to create a separate ``Limitations'' section in their paper.
        \item The paper should point out any strong assumptions and how robust the results are to violations of these assumptions (e.g., independence assumptions, noiseless settings, model well-specification, asymptotic approximations only holding locally). The authors should reflect on how these assumptions might be violated in practice and what the implications would be.
        \item The authors should reflect on the scope of the claims made, e.g., if the approach was only tested on a few datasets or with a few runs. In general, empirical results often depend on implicit assumptions, which should be articulated.
        \item The authors should reflect on the factors that influence the performance of the approach. For example, a facial recognition algorithm may perform poorly when image resolution is low or images are taken in low lighting. Or a speech-to-text system might not be used reliably to provide closed captions for online lectures because it fails to handle technical jargon.
        \item The authors should discuss the computational efficiency of the proposed algorithms and how they scale with dataset size.
        \item If applicable, the authors should discuss possible limitations of their approach to address problems of privacy and fairness.
        \item While the authors might fear that complete honesty about limitations might be used by reviewers as grounds for rejection, a worse outcome might be that reviewers discover limitations that aren't acknowledged in the paper. The authors should use their best judgment and recognize that individual actions in favor of transparency play an important role in developing norms that preserve the integrity of the community. Reviewers will be specifically instructed to not penalize honesty concerning limitations.
    \end{itemize}

\item {\bf Theory assumptions and proofs}
    \item[] Question: For each theoretical result, does the paper provide the full set of assumptions and a complete (and correct) proof?
    \item[] Answer: \answerYes{} 
    \item[] Justification: Yes assumptions and proofs are provided in Section~\ref{sec:proofs}
    \item[] Guidelines:
    \begin{itemize}
        \item The answer \answerNA{} means that the paper does not include theoretical results. 
        \item All the theorems, formulas, and proofs in the paper should be numbered and cross-referenced.
        \item All assumptions should be clearly stated or referenced in the statement of any theorems.
        \item The proofs can either appear in the main paper or the supplemental material, but if they appear in the supplemental material, the authors are encouraged to provide a short proof sketch to provide intuition. 
        \item Inversely, any informal proof provided in the core of the paper should be complemented by formal proofs provided in appendix or supplemental material.
        \item Theorems and Lemmas that the proof relies upon should be properly referenced. 
    \end{itemize}

    \item {\bf Experimental result reproducibility}
    \item[] Question: Does the paper fully disclose all the information needed to reproduce the main experimental results of the paper to the extent that it affects the main claims and/or conclusions of the paper (regardless of whether the code and data are provided or not)?
    \item[] Answer: \answerYes{} 
    \item[] Justification: Yes, the experimental detail can be found in Section~\ref{sec:experiments}.
    \item[] Guidelines:
    \begin{itemize}
        \item The answer \answerNA{} means that the paper does not include experiments.
        \item If the paper includes experiments, a \answerNo{} answer to this question will not be perceived well by the reviewers: Making the paper reproducible is important, regardless of whether the code and data are provided or not.
        \item If the contribution is a dataset and\slash or model, the authors should describe the steps taken to make their results reproducible or verifiable. 
        \item Depending on the contribution, reproducibility can be accomplished in various ways. For example, if the contribution is a novel architecture, describing the architecture fully might suffice, or if the contribution is a specific model and empirical evaluation, it may be necessary to either make it possible for others to replicate the model with the same dataset, or provide access to the model. In general. releasing code and data is often one good way to accomplish this, but reproducibility can also be provided via detailed instructions for how to replicate the results, access to a hosted model (e.g., in the case of a large language model), releasing of a model checkpoint, or other means that are appropriate to the research performed.
        \item While NeurIPS does not require releasing code, the conference does require all submissions to provide some reasonable avenue for reproducibility, which may depend on the nature of the contribution. For example
        \begin{enumerate}
            \item If the contribution is primarily a new algorithm, the paper should make it clear how to reproduce that algorithm.
            \item If the contribution is primarily a new model architecture, the paper should describe the architecture clearly and fully.
            \item If the contribution is a new model (e.g., a large language model), then there should either be a way to access this model for reproducing the results or a way to reproduce the model (e.g., with an open-source dataset or instructions for how to construct the dataset).
            \item We recognize that reproducibility may be tricky in some cases, in which case authors are welcome to describe the particular way they provide for reproducibility. In the case of closed-source models, it may be that access to the model is limited in some way (e.g., to registered users), but it should be possible for other researchers to have some path to reproducing or verifying the results.
        \end{enumerate}
    \end{itemize}

\item {\bf Open access to data and code}
    \item[] Question: Does the paper provide open access to the data and code, with sufficient instructions to faithfully reproduce the main experimental results, as described in supplemental material?
    \item[] Answer: \answerYes{} 
    \item[] Justification: We provide links to our our code and data.
    \item[] Guidelines:
    \begin{itemize}
        \item The answer \answerNA{} means that paper does not include experiments requiring code.
        \item Please see the NeurIPS code and data submission guidelines (\url{https://neurips.cc/public/guides/CodeSubmissionPolicy}) for more details.
        \item While we encourage the release of code and data, we understand that this might not be possible, so \answerNo{} is an acceptable answer. Papers cannot be rejected simply for not including code, unless this is central to the contribution (e.g., for a new open-source benchmark).
        \item The instructions should contain the exact command and environment needed to run to reproduce the results. See the NeurIPS code and data submission guidelines (\url{https://neurips.cc/public/guides/CodeSubmissionPolicy}) for more details.
        \item The authors should provide instructions on data access and preparation, including how to access the raw data, preprocessed data, intermediate data, and generated data, etc.
        \item The authors should provide scripts to reproduce all experimental results for the new proposed method and baselines. If only a subset of experiments are reproducible, they should state which ones are omitted from the script and why.
        \item At submission time, to preserve anonymity, the authors should release anonymized versions (if applicable).
        \item Providing as much information as possible in supplemental material (appended to the paper) is recommended, but including URLs to data and code is permitted.
    \end{itemize}

\item {\bf Experimental setting/details}
    \item[] Question: Does the paper specify all the training and test details (e.g., data splits, hyperparameters, how they were chosen, type of optimizer) necessary to understand the results?
    \item[] Answer: \answerYes{} 
    \item[] Justification: Yes, data splits are provided in Section~\ref{sec:experiments}, while all implementation details are reported in Section~\ref{sec:implementation}.
    \item[] Guidelines:
    \begin{itemize}
        \item The answer \answerNA{} means that the paper does not include experiments.
        \item The experimental setting should be presented in the core of the paper to a level of detail that is necessary to appreciate the results and make sense of them.
        \item The full details can be provided either with the code, in appendix, or as supplemental material.
    \end{itemize}

\item {\bf Experiment statistical significance}
    \item[] Question: Does the paper report error bars suitably and correctly defined or other appropriate information about the statistical significance of the experiments?
    \item[] Answer: \answerNo{} 
    \item[] Justification: We follow the same procedures as previous works~\cite{ganea2021geomol,hassan2024flow,frank2025sampling}.
    \item[] Guidelines:
    \begin{itemize}
        \item The answer \answerNA{} means that the paper does not include experiments.
        \item The authors should answer \answerYes{} if the results are accompanied by error bars, confidence intervals, or statistical significance tests, at least for the experiments that support the main claims of the paper.
        \item The factors of variability that the error bars are capturing should be clearly stated (for example, train/test split, initialization, random drawing of some parameter, or overall run with given experimental conditions).
        \item The method for calculating the error bars should be explained (closed form formula, call to a library function, bootstrap, etc.)
        \item The assumptions made should be given (e.g., Normally distributed errors).
        \item It should be clear whether the error bar is the standard deviation or the standard error of the mean.
        \item It is OK to report 1-sigma error bars, but one should state it. The authors should preferably report a 2-sigma error bar than state that they have a 96\% CI, if the hypothesis of Normality of errors is not verified.
        \item For asymmetric distributions, the authors should be careful not to show in tables or figures symmetric error bars that would yield results that are out of range (e.g., negative error rates).
        \item If error bars are reported in tables or plots, the authors should explain in the text how they were calculated and reference the corresponding figures or tables in the text.
    \end{itemize}

\item {\bf Experiments compute resources}
    \item[] Question: For each experiment, does the paper provide sufficient information on the computer resources (type of compute workers, memory, time of execution) needed to reproduce the experiments?
    \item[] Answer: \answerYes{} 
    \item[] Justification: Yes, necessary details are provided in Section~\ref{sec:implementation}.
    \item[] Guidelines:
    \begin{itemize}
        \item The answer \answerNA{} means that the paper does not include experiments.
        \item The paper should indicate the type of compute workers CPU or GPU, internal cluster, or cloud provider, including relevant memory and storage.
        \item The paper should provide the amount of compute required for each of the individual experimental runs as well as estimate the total compute. 
        \item The paper should disclose whether the full research project required more compute than the experiments reported in the paper (e.g., preliminary or failed experiments that didn't make it into the paper). 
    \end{itemize}
    
\item {\bf Code of ethics}
    \item[] Question: Does the research conducted in the paper conform, in every respect, with the NeurIPS Code of Ethics \url{https://neurips.cc/public/EthicsGuidelines}?
    \item[] Answer: \answerYes{} 
    \item[] Justification: We follow the the NeurIPS Code of Ethics.
    \item[] Guidelines:
    \begin{itemize}
        \item The answer \answerNA{} means that the authors have not reviewed the NeurIPS Code of Ethics.
        \item If the authors answer \answerNo, they should explain the special circumstances that require a deviation from the Code of Ethics.
        \item The authors should make sure to preserve anonymity (e.g., if there is a special consideration due to laws or regulations in their jurisdiction).
    \end{itemize}

\item {\bf Broader impacts}
    \item[] Question: Does the paper discuss both potential positive societal impacts and negative societal impacts of the work performed?
    \item[] Answer: \answerNo{} 
    \item[] Justification: This work aims to advance the field of AI for Science. While research in this area may have broader societal implications, we are not aware of any specific societal consequences that require explicit discussion in the context of this work.
    \item[] Guidelines:
    \begin{itemize}
        \item The answer \answerNA{} means that there is no societal impact of the work performed.
        \item If the authors answer \answerNA{} or \answerNo, they should explain why their work has no societal impact or why the paper does not address societal impact.
        \item Examples of negative societal impacts include potential malicious or unintended uses (e.g., disinformation, generating fake profiles, surveillance), fairness considerations (e.g., deployment of technologies that could make decisions that unfairly impact specific groups), privacy considerations, and security considerations.
        \item The conference expects that many papers will be foundational research and not tied to particular applications, let alone deployments. However, if there is a direct path to any negative applications, the authors should point it out. For example, it is legitimate to point out that an improvement in the quality of generative models could be used to generate Deepfakes for disinformation. On the other hand, it is not needed to point out that a generic algorithm for optimizing neural networks could enable people to train models that generate Deepfakes faster.
        \item The authors should consider possible harms that could arise when the technology is being used as intended and functioning correctly, harms that could arise when the technology is being used as intended but gives incorrect results, and harms following from (intentional or unintentional) misuse of the technology.
        \item If there are negative societal impacts, the authors could also discuss possible mitigation strategies (e.g., gated release of models, providing defenses in addition to attacks, mechanisms for monitoring misuse, mechanisms to monitor how a system learns from feedback over time, improving the efficiency and accessibility of ML).
    \end{itemize}
    
\item {\bf Safeguards}
    \item[] Question: Does the paper describe safeguards that have been put in place for responsible release of data or models that have a high risk for misuse (e.g., pre-trained language models, image generators, or scraped datasets)?
    \item[] Answer: \answerNA{} 
    \item[] Justification: This paper presents research aimed at advancing the field of AI for Science. While potential risks of misuse exist, we are not aware of any that require specific safeguards.
    \item[] Guidelines:
    \begin{itemize}
        \item The answer \answerNA{} means that the paper poses no such risks.
        \item Released models that have a high risk for misuse or dual-use should be released with necessary safeguards to allow for controlled use of the model, for example by requiring that users adhere to usage guidelines or restrictions to access the model or implementing safety filters. 
        \item Datasets that have been scraped from the Internet could pose safety risks. The authors should describe how they avoided releasing unsafe images.
        \item We recognize that providing effective safeguards is challenging, and many papers do not require this, but we encourage authors to take this into account and make a best faith effort.
    \end{itemize}

\item {\bf Licenses for existing assets}
    \item[] Question: Are the creators or original owners of assets (e.g., code, data, models), used in the paper, properly credited and are the license and terms of use explicitly mentioned and properly respected?
    \item[] Answer: \answerYes{} 
    \item[] Justification: Yes, we cite all creators and original owners of work referenced in this manuscript.
    \item[] Guidelines:
    \begin{itemize}
        \item The answer \answerNA{} means that the paper does not use existing assets.
        \item The authors should cite the original paper that produced the code package or dataset.
        \item The authors should state which version of the asset is used and, if possible, include a URL.
        \item The name of the license (e.g., CC-BY 4.0) should be included for each asset.
        \item For scraped data from a particular source (e.g., website), the copyright and terms of service of that source should be provided.
        \item If assets are released, the license, copyright information, and terms of use in the package should be provided. For popular datasets, \url{paperswithcode.com/datasets} has curated licenses for some datasets. Their licensing guide can help determine the license of a dataset.
        \item For existing datasets that are re-packaged, both the original license and the license of the derived asset (if it has changed) should be provided.
        \item If this information is not available online, the authors are encouraged to reach out to the asset's creators.
    \end{itemize}

\item {\bf New assets}
    \item[] Question: Are new assets introduced in the paper well documented and is the documentation provided alongside the assets?
    \item[] Answer: \answerYes{} 
    \item[] Justification: We plan to release our code together with appropriate documentation. These materials are currently being prepared and will be made available in a future revision of the paper.
    \item[] Guidelines:
    \begin{itemize}
        \item The answer \answerNA{} means that the paper does not release new assets.
        \item Researchers should communicate the details of the dataset\slash code\slash model as part of their submissions via structured templates. This includes details about training, license, limitations, etc. 
        \item The paper should discuss whether and how consent was obtained from people whose asset is used.
        \item At submission time, remember to anonymize your assets (if applicable). You can either create an anonymized URL or include an anonymized zip file.
    \end{itemize}

\item {\bf Crowdsourcing and research with human subjects}
    \item[] Question: For crowdsourcing experiments and research with human subjects, does the paper include the full text of instructions given to participants and screenshots, if applicable, as well as details about compensation (if any)? 
    \item[] Answer: \answerNA{} 
    \item[] Justification: No crowdsourcing or research with human subjects is done in our work.
    \item[] Guidelines:
    \begin{itemize}
        \item The answer \answerNA{} means that the paper does not involve crowdsourcing nor research with human subjects.
        \item Including this information in the supplemental material is fine, but if the main contribution of the paper involves human subjects, then as much detail as possible should be included in the main paper. 
        \item According to the NeurIPS Code of Ethics, workers involved in data collection, curation, or other labor should be paid at least the minimum wage in the country of the data collector. 
    \end{itemize}

\item {\bf Institutional review board (IRB) approvals or equivalent for research with human subjects}
    \item[] Question: Does the paper describe potential risks incurred by study participants, whether such risks were disclosed to the subjects, and whether Institutional Review Board (IRB) approvals (or an equivalent approval/review based on the requirements of your country or institution) were obtained?
    \item[] Answer: \answerNA{} 
    \item[] Justification: No IRB Approvals or equivalent for research with human subjects are done in our work.
    \item[] Guidelines:
    \begin{itemize}
        \item The answer \answerNA{} means that the paper does not involve crowdsourcing nor research with human subjects.
        \item Depending on the country in which research is conducted, IRB approval (or equivalent) may be required for any human subjects research. If you obtained IRB approval, you should clearly state this in the paper. 
        \item We recognize that the procedures for this may vary significantly between institutions and locations, and we expect authors to adhere to the NeurIPS Code of Ethics and the guidelines for their institution. 
        \item For initial submissions, do not include any information that would break anonymity (if applicable), such as the institution conducting the review.
    \end{itemize}

\item {\bf Declaration of LLM usage}
    \item[] Question: Does the paper describe the usage of LLMs if it is an important, original, or non-standard component of the core methods in this research? Note that if the LLM is used only for writing, editing, or formatting purposes and does \emph{not} impact the core methodology, scientific rigor, or originality of the research, declaration is not required.
    \item[] Answer: \answerNA{} 
    \item[] Justification: The core methodology does not involve the use of LLMs. 
    \item[] Guidelines:
    \begin{itemize}
        \item The answer \answerNA{} means that the core method development in this research does not involve LLMs as any important, original, or non-standard components.
        \item Please refer to our LLM policy in the NeurIPS handbook for what should or should not be described.
    \end{itemize}

\end{enumerate}

\end{document}